\documentclass{article}

\usepackage{microtype}
\usepackage{graphicx}
\usepackage{subfigure}
\usepackage{booktabs} 
\usepackage{xcolor}
\usepackage{soul}
\usepackage{epsfig}
\usepackage{graphicx}
\usepackage{subfigure}
\usepackage{booktabs}
\usepackage{caption}
\usepackage{array}
\usepackage{etoolbox}
\usepackage{tabulary}
\usepackage{colortbl}
\usepackage{bbding}
\usepackage{wrapfig}
\usepackage{multirow}
\usepackage{soul}
\usepackage{amsmath}
\usepackage{hyperref}
\usepackage{tabularx}
\usepackage{amssymb}
\usepackage{multirow}
\usepackage{graphicx}
\usepackage{caption} 
\definecolor{highlightcolor}{RGB}{238, 242, 250} 

\usepackage{algorithm}
\usepackage{algpseudocode}
\renewcommand{\paragraph}[1]{\vspace{1.25mm}\noindent\textbf{#1}}

\usepackage[accepted]{icml2024}

\usepackage{amsmath}
\usepackage{amssymb}
\usepackage{mathtools}
\usepackage{amsthm}

\usepackage[capitalize,noabbrev]{cleveref}

\theoremstyle{plain}
\newtheorem{theorem}{Theorem}[section]

\newtheorem{corollary}[theorem]{Corollary}
\theoremstyle{definition}
\newtheorem{definition}[theorem]{Definition}

\theoremstyle{remark}

\usepackage{paralist}

\usepackage[textsize=tiny]{todonotes}

\icmltitlerunning{Graph Condensation via Expanding Window Matching}

\begin{document}

\twocolumn[
\icmltitle{Navigating Complexity: Toward Lossless Graph Condensation via Expanding Window Matching}



\icmlsetsymbol{equal}{*}

\begin{icmlauthorlist}
\icmlauthor{Yuchen Zhang}{equal,yyy}
\icmlauthor{Tianle Zhang}{equal,yyy}
\icmlauthor{Kai Wang$^\dag$}{yyy}
\icmlauthor{Ziyao Guo}{yyy}
\icmlauthor{Yuxuan Liang}{comp}
\icmlauthor{Xavier Bresson}{yyy}
\icmlauthor{Wei Jin}{sch}
\icmlauthor{Yang You}{yyy}
\end{icmlauthorlist}

\icmlaffiliation{yyy}{National University of Singapore}
\icmlaffiliation{sch}{Emory University}
\icmlaffiliation{comp}{Hong Kong University of Science and Technology (Guangzhou)}

\icmlcorrespondingauthor{Wei Jin}{wei.jin@emory.edu}
\icmlcorrespondingauthor{Yang You}{youy@comp.nus.edu.sg}

\icmlkeywords{Machine Learning, ICML}

\vskip 0.3in
]



\printAffiliationsAndNotice{\icmlEqualContribution, \icmlProjectLead} %

\begin{abstract}
Graph condensation aims to reduce the size of a large-scale graph dataset by synthesizing a compact counterpart without sacrificing the performance of Graph Neural Networks (GNNs) trained on it, which has shed light on reducing the computational cost for training GNNs. 
Nevertheless, existing methods often fall short of accurately replicating the original graph for certain datasets, thereby failing to achieve the objective of lossless condensation. 
To understand this phenomenon, we investigate the potential reasons and reveal that the previous state-of-the-art trajectory matching method provides biased and restricted supervision signals from the original graph when optimizing the condensed one. This significantly limits both the scale and efficacy of the condensed graph.
In this paper, we make the first attempt toward \textit{lossless graph condensation} by bridging the previously neglected supervision signals. 
Specifically, we employ a curriculum learning strategy to train expert trajectories with more diverse supervision signals from the original graph, and then effectively transfer the information into the condensed graph with expanding window matching.
Moreover, we design a loss function to further extract knowledge from the expert trajectories. 
Theoretical analysis justifies the design of our method and extensive experiments verify its superiority across different datasets. 
Code is released at \href{https://github.com/NUS-HPC-AI-Lab/GEOM}{https://github.com/NUS-HPC-AI-Lab/GEOM}.
\end{abstract}

\section{Introduction} \label{introduction}
Graph condensation follows the success in vision dataset distillation~\cite{wang2018dataset,zhao2020dataset, nguyen2021dataset, cazenavette2022dataset,zhou2022dataset,zhou2023dataset} and aims to synthesize a smaller condensed graph dataset from the original one. 
Recently, gradient and trajectory matching methods~\cite{jin2022condensing,jin2022graph,zheng2023structure,hashemi2024comprehensive} have achieved remarkable results on some small-scale graph datasets. 
For instance, SFGC~\cite{zheng2023structure} condenses Citeseer~\cite{kipf2016semi} to 1.8\% sparsity without performance drop.
However, these methods fail to perform well on large-scale graph datasets, \textit{i.e}, there persists an unignorable performance gap between GNNs trained on the condensed and original graph datasets.
This severely limits their effectiveness in real-world scenarios.
Therefore, developing a high-performing and robust graph condensation approach has become urgent for broader graph-related applications.

In addressing the condensation challenges on large-scale graphs, a critical question arises: \textit{what causes the notable discrepancy in performance between condensing large-scale and small-scale graphs}?
To this end, we analyze the differing outcomes of previous methods applied to graphs of various sizes. 
A key observation highlights that a smaller condensation ratio is utilized for large-scale graphs compared to small-scale ones.
This suggests a greater disparity in size between large-scale graphs and their condensed counterparts. 
One intuitive solution is to enlarge the condensation ratio.
We conduct experiments with the existing methods and show results in Fig.~\ref{fig:Increased Accuracy}.
However, we find that the performance of the condensed graph saturates as the ratio increases. 
Besides, a significant gap still exists between the saturated performance and that of the original graph.

\begin{figure*}[t]
\centering
    \subfigure[]
    {\includegraphics[width=0.226\textwidth, angle=0]{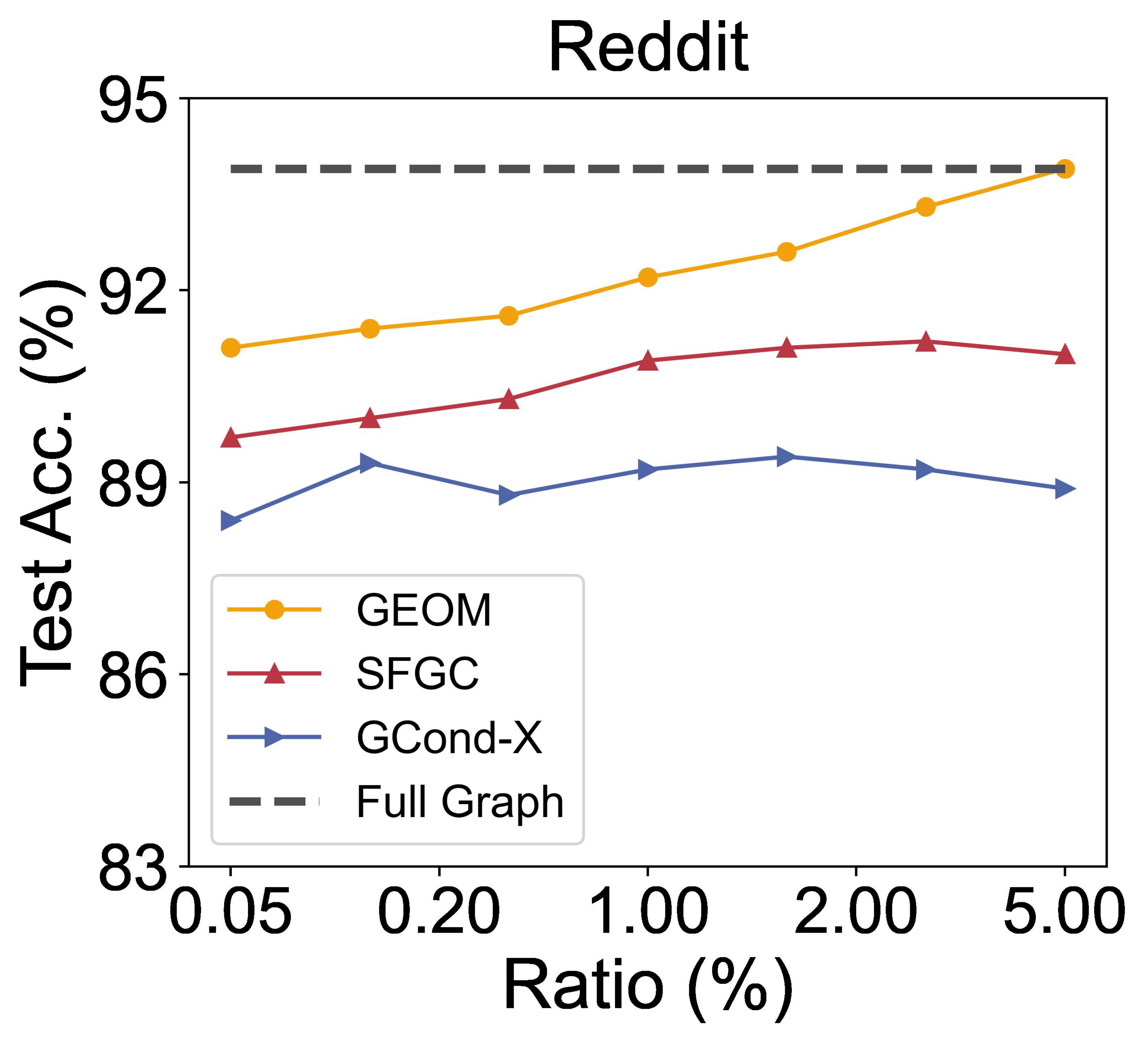}
    \label{fig:Increased Accuracy}}
    \hspace{1mm}
    \subfigure[]
    {\includegraphics[width=0.23\textwidth, angle=0]{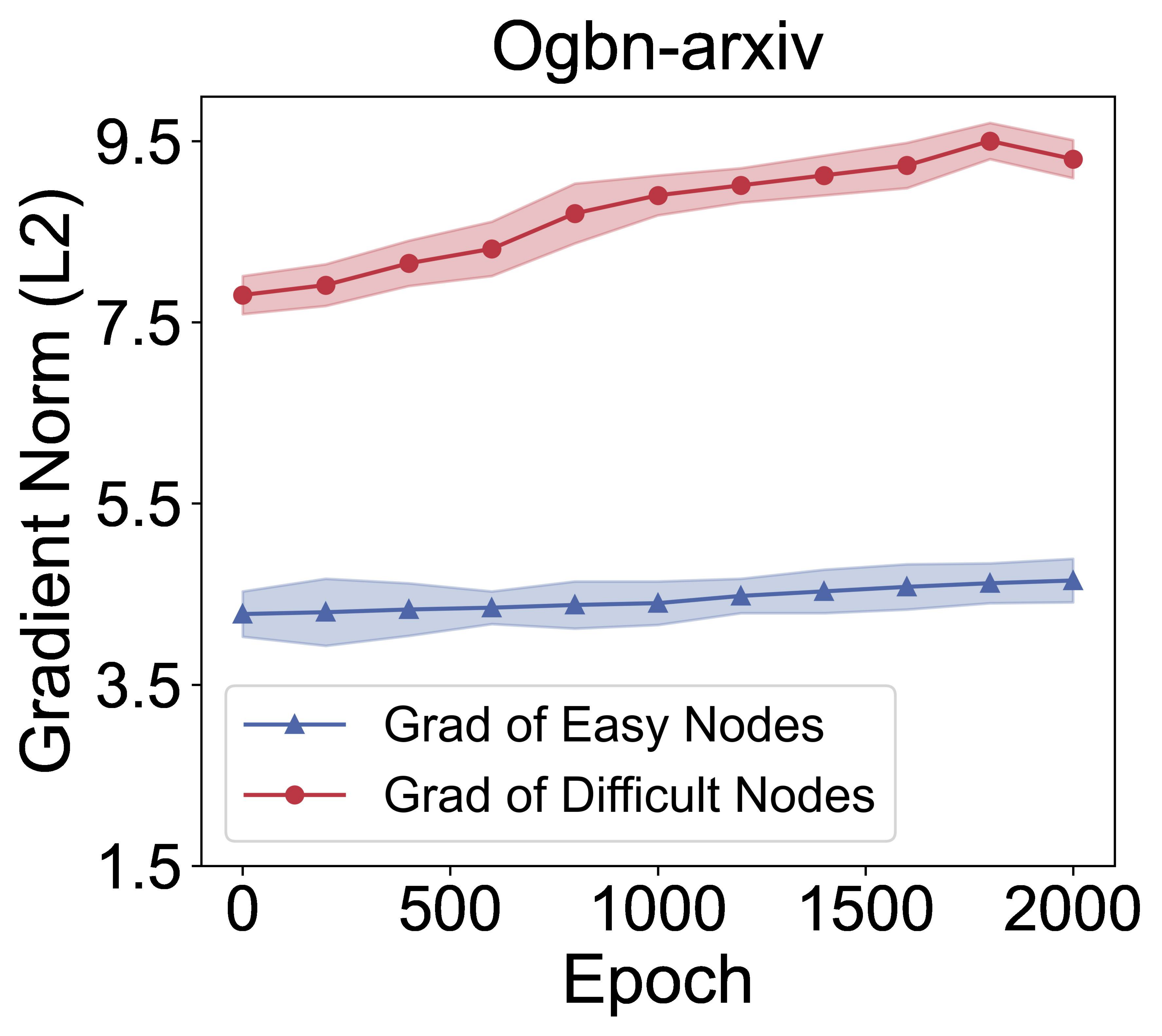}
    \label{fig:arxiv grad of SFGC}}
    \hspace{1mm}
    \subfigure[]
    {\includegraphics[width=0.226\textwidth, angle=0]{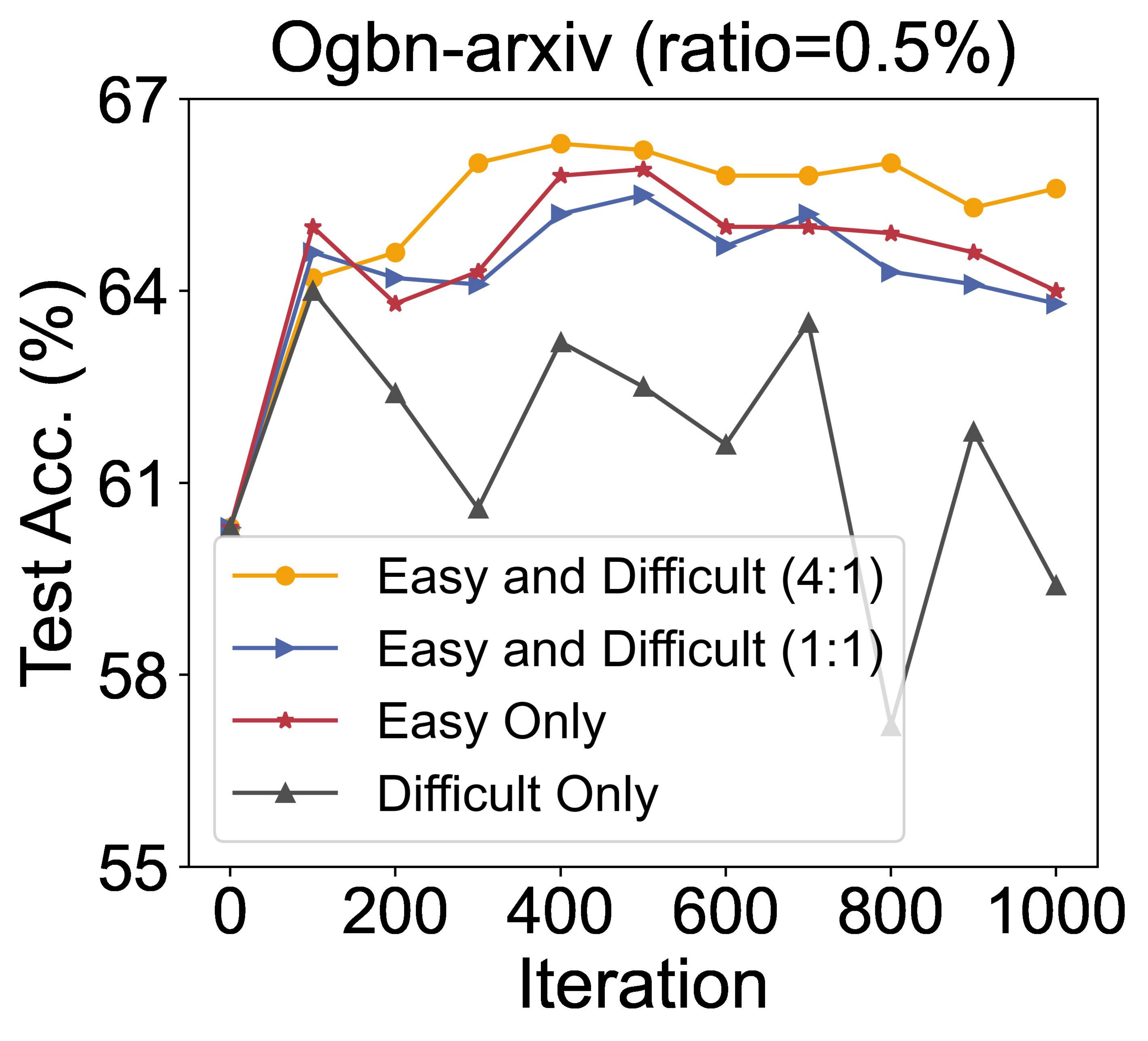}
    \label{fig:Performance of different nodes1}}
    \hspace{1mm}
    \subfigure[]
    {\includegraphics[width=0.226\textwidth, angle=0]{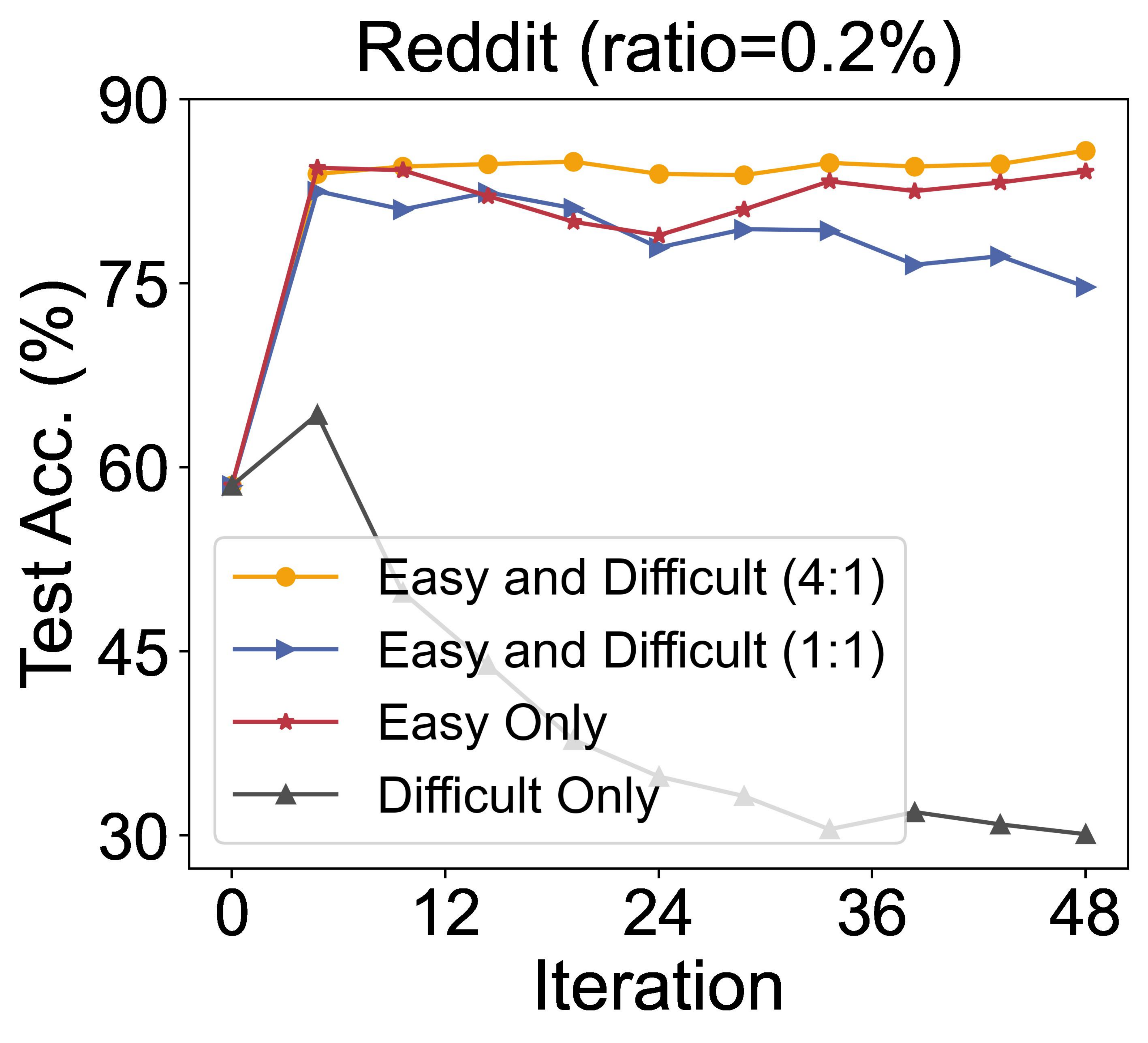}
    \label{fig:Performance of different nodes2}}
    \vspace{-10pt}
\caption{(a) shows that the performances in previous methods~\cite{jin2022graph, zheng2023structure} stop increasing after the condensation ratio in large-scale graph dataset reaches a certain threshold. (b) shows that the gradient generated by easy nodes and difficult nodes during the training process of the experts. (c) and (d) 
indicate that using supervision signals provided by solely easy nodes can yield better performance for the condensed graph than solely using difficult nodes. Nevertheless, incorporating a small portion of difficult nodes into easy ones can further enhance the performance of the condensed graph.}\label{fig:pre-exp}
    \vspace{-8.0pt}
\end{figure*}

GCond~\cite{jin2022graph} explains this phenomenon as follows: optimizing a larger condensed graph might be more complex. 
Nevertheless, GCond does not specify the exact cause of this phenomenon.
We infer that the lack of rich supervision from the original graph might be a potential reason.
To verify it, we take trajectory matching method SFGC~\cite{zheng2023structure} as an example\footnote{We provide detailed experimental settings in Appendix~\ref{section: samples}} for the following exploration.
SFGC trains a set of expert trajectories on the original graph as supervision for optimizing the condensed graph.
Thus, to analyze the principal components of the supervision, we visualize the gradient norm of easy and difficult nodes (categorized by homophily level~\cite{wei2023clnode}) from the original graph in Fig.~\ref{fig:arxiv grad of SFGC}. 
One can find that the difficult nodes dominate the principal components of supervision from the original graph. 
This may cause the condensed graph to exhibit a bias toward difficult nodes while overlooking easy nodes~\cite{wang2022cafe}.

To investigate the effectiveness of different components of the supervision signals, we select four ratios of mixed easy and difficult nodes to train the expert trajectories. 
We report the performance comparisons of the condensed graph in Fig.~\ref{fig:Performance of different nodes1} and~\ref{fig:Performance of different nodes2}. 
Several observations can be concluded as follows: 1) Only using the supervision signals from difficult nodes can not perform well; 2) Using a proper ratio of easy and difficult nodes obtains the best performance in our experiments.
These observations highlight the importance of understanding the impact of node characteristics on learning dynamics in the condensed graph.
In graph learning, easy nodes have representative features of the original graph, whereas difficult nodes contain ambiguous features~\cite{wei2023clnode}.
Combined with the observations, we conclude that condensed graph captures representative patterns
under the supervision of easy nodes.
Although proper supervision of difficult nodes further enrich the patterns, excessive supervision of them may result in chaotic features~\cite{liu2023dream,liu2023dream+} and damage the representative patterns.



Based on our findings, we propose a novel approach called \textbf{G}raph Condensation via \textbf{E}xpanding Wind\textbf{O}w \textbf{M}atching (GEOM).
Specifically, we train the expert trajectories with curriculum learning to involve more diverse supervision signals from the original graph.
Then, we utilize an expanding window to determine the matching range when matching trajectories.
In this way, we enable the rich information of the original graph can be compactly and efficiently transferred to the condensed counterpart. Theoretical analysis justifies our design from the perspective of reducing accumulated error~\cite{du2023minimizing}.
Furthermore, inspired by network distillation~\cite{hinton2015distilling, zhang2021graph}, we design a loss function to uncover information in expert trajectories from a new perspective.

In this work, we make the first attempt toward \textit{lossless graph condensation}. 
Concretely, we condense Citeseer to 0.9\%, Cora to 1.3\%, Ogbn-arxiv to 5\%, Flickr to 1\%, and Reddit to 5\% without any performance loss when training a GCN. 
Moreover, our condensed graphs can generalize well to different GNN models, and even achieve \textit{lossless} performance across 20 out of 35 cross-architecture experiments. 
We hope our work can help mitigate the heavy computation cost for training GNNs on large-scale graph datasets and broaden the real-world applications of graph condensation.

\vspace{-0.3em}
\section{Method} \label{method}
\vspace{-0.3em}
In this section, we first briefly overview the framework of trajectory matching graph condensation and curriculum learning. Then we introduce the components of our method as well as theoretical understanding.

\subsection{Preliminaries} \label{2.1}
\textbf{Trajectory matching graph condensation}~\cite{zheng2023structure}\textbf{.}
Given a large graph dataset $\mathcal{T}$, trajectory matching graph condensation synthesizes a small graph dataset $\mathcal{S}$ by minimizing the training trajectory distance on $\mathcal{T}$ and $\mathcal{S}$. 
It aims to reduce the performance gap between GNNs trained on $\mathcal{T}$ and $\mathcal{S}$.
Generally, trajectory matching graph condensation can be divided into three phases. 
\begin{compactenum}[(a)]
\item \textbf{Buffer Phase}.
Preparing the expert trajectories: training GNNs on $\mathcal{T}$ and saving the checkpoints. 
\item  \textbf{Condensation Phase}. Condensing the original graph dataset: optimizing the condensed graph by matching the training trajectories between $\mathcal{T}$ and $\mathcal{S}$.
\item \textbf{Evaluation Phase}. Evaluating the condensed graph dataset: using the condensed graph datasets to train a randomly initialized GNN.
\end{compactenum}
Formally, in the \textit{condensation} phase, we optimize the following objective to synthesize the condensed graph $\mathcal{S}$:
\begin{equation}\label{eqn-0}
\min_{S} \mathrm{E}_{\theta_{t}^{*}\sim P_{\theta_{\tau}}} \left[ \mathcal{L}_{M} \left( (\theta_{t}^{*}|^{t+p}_{t}, \widetilde{\theta}_{t}|^{t+q}_{t})\right) \right],
\end{equation}
where $\theta_{t}^{*}|^{t+p}_{t}, \widetilde{\theta}_{t}|^{t+q}_{t}$ 
denote the parameters of GNNs trained on $\mathcal{T}$ and $\mathcal{S}$ within checkpoints $\left(t,t+p\right)$, 
$P_{\theta_{\tau}} $ denotes the parameter distribution with the expert trajectories.
$\mathcal{L}_{M}$ is the distance between trajectories trained on $\mathcal{T}$ and $\mathcal{S}$, which can be written as:
\begin{equation}\label{eqn-1}
\mathcal{L}_{M} = \frac { \left \| \widetilde{\theta}_{t+q}-\theta_{t+p}^{*}\right\|_2^2}{\left\|\widetilde{\theta}_{t}-\theta_{t+p}^{*}\right\|_2^2},
\end{equation}
where $\widetilde{\theta}_{t} = \theta_{t}^{*}$, $\theta_{t+p}^{*}$ denotes the model parameters $p$ checkpoints after $\theta_{t}^{*}$. Meanwhile, $\widetilde{\theta}_{t+q}$ results from $q$ inner-loops using the classification loss $\ell$, applied to dataset $\mathcal{S}$, and a learnable learning rate $\eta$:
\begin{equation}\label{eqn-3}
 \widetilde{\theta}_{t+i+1}= \widetilde{\theta}_{t+i}-\eta\nabla\ell\left(f\left(\widetilde{\theta}_{t+i};\mathcal{S}\right),\mathcal{Y}\right),
\end{equation}
where $f\left( ; \right)$ is the GNN trained on $\mathcal{S}$, $\mathcal{Y}$ is the label set of the condensed graph dataset.
Note that in the \textit{buffer} phase, GNN is trained on the whole graph dataset $\mathcal{T}$ by default.

\textbf{Curriculum learning.} 
Different from the normal training scheme~\cite{zheng2023structure, cazenavette2022dataset}, the most distinctive characteristic of curriculum learning (CL) lies in differentiating the training samples~\cite{bengio2009curriculum, krueger2009flexible}.
Specifically, CL imitates how humans learn by organizing data samples in a logical sequence, primarily from easy to difficult, as the curriculum for model training~\cite{wei2023clnode, wang2021curgraph}.
Prior works demonstrate that CL steers models to a more optimal parameter space~\cite{li2023curriculum,bengio2009curriculum} than normal training.
CL can enhance model performance, generalization, robustness, and even convergence in diverse scenarios~\cite{sitawarin2021sat,weinshall2020theory, krishnapriyan2021characterizing}.


\subsection{Preparing Curriculum-based Expert Trajectories} \label{2.2}
The superiority of CL has been demonstrated across various tasks, prompting us to integrate CL into graph condensation.
Taking a closer look at CL, it allows the model to initially focus on easy samples and then gradually shift attention to more difficult ones, thereby forming expert trajectories with more diverse supervision signals.
To implement the CL approach, we design a difficulty measurer based on homophily to differentiate between easy and difficult samples.
Moreover, we utilize a continuous training scheduler to sequence the samples into an easy-to-difficult curriculum.

\textbf{Homophily-based difficulty measurer.} 
On node classification tasks, GNNs learn node representation through an iterative process of aggregating neighborhood information~\cite{ma2021homophily, halcrow2020grale}.
Owing to this mechanism, the nodes tend to aggregate features from neighbors sharing the same class will receive additional information about their class features. 
Thus, GNNs are more adept at learning these nodes as they will have more representative features~\cite{chien2020adaptive,zhu2020beyond}.
Conversely, for nodes aggregate features from neighbors in many different classes, their representations become chaotic, making them hard to learn~\cite{maurya2021improving,mao2023demystifying}.


Thus, inspired by CLNode~\cite{wei2023clnode}, we calculate the difficulty score for each node through the label distribution of its neighborhood to distinguish between easy and difficult. 
Specifically, for each training node $x$, the difficulty score can be calculated as follows:
\begin{equation}
\mathcal{P}_{c}(x)=\frac{|\{y_n=c|n\in\mathcal{N}(x)\cup\{x\}\wedge y_n \in \mathcal{Y}\}|}{|{\mathcal N}(x)\cup\{x\}|},
\end{equation}
\begin{equation}\label{eqn-4}
\mathcal{D}(x)=-\sum_{c\in C}P_c(x)log(P_c(x)),
\end{equation}
Where $y_n$ denotes the label of node $n$, $\mathcal{P}_{c}(x)$ represents the proportion of neighborhood nodes $\mathcal{N}(x)\cup x$ in class $c$. The difficulty score $\mathcal{D}(x)$ is higher as the neighbor nodes of node $x$ become more diverse (as illustrated in Fig~\ref{fig:diff}). 

\begin{figure}[htbp]
 \centering
 \includegraphics[width=0.42\textwidth]{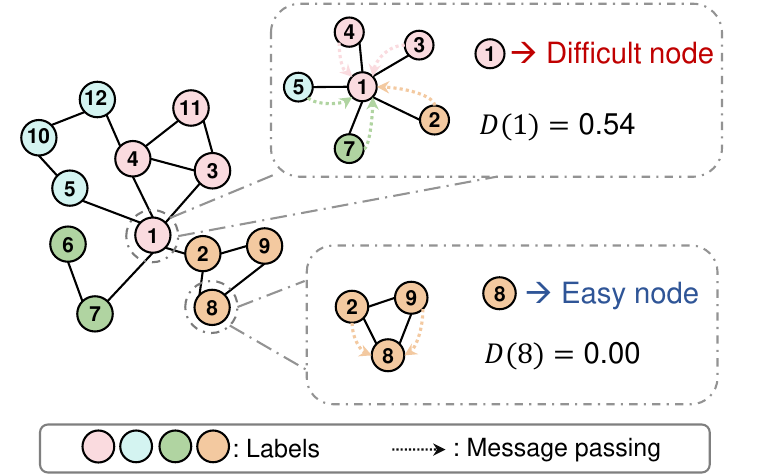}
\caption{An example of homophily-based difficulty measurer.}
\label{fig:diff}
\end{figure}

\textbf{Curriculum training scheduler.} 
After getting the difficulty score, we utilize a continuous training scheduler to generate an easy-to-difficult curriculum for training expert trajectories. 
Specifically, we use a pacing function to map each epoch $t$ to a scalar $\lambda_t$ in (0, 1], and then select a proportion $\lambda_t$ of the easiest nodes for training at epoch $t$.
More details of the pacing function is detailed in Appendix~\ref{training sch}. 
Furthermore, we do not stop training when the whole graph training set is involved, as the recently added nodes may not have been sufficiently learned at this time. 
Specifically, we persist in training with the whole graph training set until the validation set accuracy converges.

\begin{figure*}[htbp]
 \centering
 \includegraphics[width=0.9\textwidth]{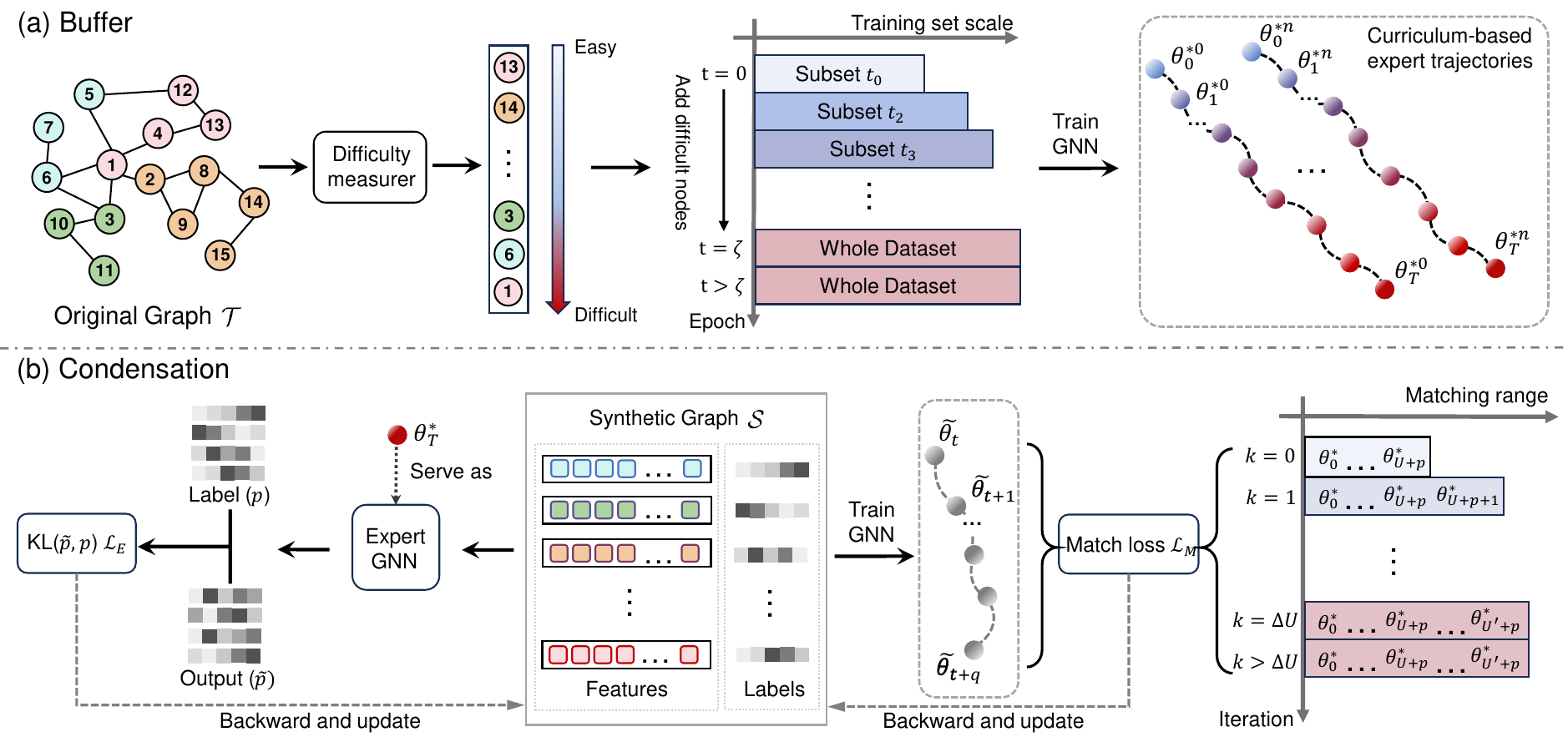}
   \vspace{-9.0pt}
\caption{Overall framework of GEOM. In the \textit{buffer} phase, we train the expert trajectories with curriculum learning to involve more informative supervision signals from the original graph. In the \textit{condensation} phase, we utilize expanding window matching to capture the rich information. Moreover, a knowledge embedding extractor is used to further extract knowledge from the expert trajectories.}
\vspace{-8.0pt}
 \label{fig:framework}
\end{figure*}

\subsection{Expanding Window Matching} 
Since we obtain expert trajectories with more diverse supervision signals through CL in the \textit{buffer} phase, 
we aim to fully utilize the rich information embedded in them to optimize the condensed graph. 
One straightforward way is to gradually move the matching range (we can sample the expert trajectory segments that need to be matched from this range) later to shift the focus of the condensed graph from primarily learning from easy nodes to difficult nodes. 

However, such a matching strategy significantly degrades the performance of the condensed graph (as shown in Table~\ref{tab:abl}).
One potential reason is: once the whole matching range is shifted later, the condensed graph falls into the trap of learning patterns from the difficult nodes continually, thereby collapsing the representative patterns.

To address the challenge of effectively learning patterns from easy and difficult nodes, we propose to use an adaptive window that gradually expands the matching range instead of a fixed sliding window, termed expanding window matching.
Formally, we determine the matching range $\mathcal{R}$ as:
\begin{equation}\label{eqn-5}
\mathcal{R} =\begin{cases}
\{\theta_0^{*}\ldots\theta_{U+p}^{*}\} ,&\quad {I} <{U}\\
\{\theta_0^{*}\ldots\theta_{U+p}^{*}\; \theta_{U+p+I}^{*} \},&\quad {I}\geq{U}\wedge {I}<{U^{'}},\\
\{\theta_0^{*}\ldots\theta_{U+p}^{*} \ldots \theta_{U^{\prime}+p}^{*}\},&\quad {I}\geq{U^{'}}
\end{cases}
\end{equation}
where $I$ denotes the number of the iteration in the \textit{condensation} phase, ${U}$ and ${U^{'}}$ are two different upperbounds to control the size of the matching range. 

Expanding window matching ensures that in the early stages of \textit{condensation}, the main component of supervision signals is from easy nodes, allowing the condensed graph to initially learn representative patterns. 
In the later stages of \textit{condensation}, the condensed graph can maintain these representative patterns while enriching them. 
This is because the matching strategy efficiently controls the weight of easy and difficult nodes in supervision signals, the condensed graph has the opportunity to learn from both of them.

From another perspective, in the previous method, the matching range is confined to a very narrow scope, causing only a few checkpoints can be utilized. 
In contrast, the proposed expanding window matching brings 10 times more available checkpoints than before, thereby more effectively utilizing the information provided by expert trajectories. 

Next, we provide the theoretical analysis to demonstrate the advantages of employing CL in the \textit{buffer} phase and the expanding window matching in the \textit{condensation} phase.

\textbf{Theoretical understanding.}
In the \textit{condensation} phase, the trajectory on $\mathcal{S}$ is optimized to reproduce the trajectory on $\mathcal{T}$ with $\widetilde{\theta}_{t} = \theta_{t}^{*}$. However, in the \textit{evaluation} phase, the starting points are no longer initialized by the parameters on $\mathcal{T}$ and the parameters are continually updated by $\mathcal{S}$. 
The error accumulates progressively in the \textit{evaluation} phase,
which leads to greater divergence between the final parameters of GNNs trained on $\mathcal{S}$ and $\mathcal{T}$.
Thus, reducing this error helps improve the final performance of $\mathcal{S}$.

Following~\citet{du2023minimizing}, we first divide the training trajectories into $N$ stages to be matched, denoted as $\{\theta^{*}_{0,0},...,\theta^{*}_{0,p},\theta^{*}_{1,0},...,\theta^{*}_{N-1,p}\}$,  
where the last parameter of a previous stage is the starting parameter of the next stage, \textit{i.e.} \(\theta^*_{n,0}=\theta^*_{n-1,p}\).
The training trajectory of GNNs trained on $\mathcal{S}$ can also be divided into corresponding $N$ stages similarly. 
For any given stage \(n\), the following definition apply:
\begin{definition}
\textit{Accumulated error $\epsilon_n$} refers to the difference in model parameters trained on condensed and original graphs at stage $n$ 
during the \textit{evaluation} phase:
\begin{equation}\label{eqn-6}
\epsilon_n = \widetilde{\theta}_{n,q} - \theta^{*}_{n,p} = \widetilde{\theta}_{n+1,0} - \theta^{*}_{n+1,0},
\end{equation}
\end{definition}
To specifically analyze the accumulated error, we introduce two additional error terms as follows:

\begin{definition}
\textit{Initialization error $I$} refers to the discrepancies caused by varying initial parameters during the training process. 
Specifically, even if the condensed graph can generate identical trajectories to the original, variations at the trajectories' endpoints are unavoidable due to the differing starting points in the \textit{condensation} phase compared to the \textit{evaluation} phase, \textit{i.e.}
$\widetilde\theta_{n,0}=\theta^{*}_{n,0}+\epsilon_{n-1}$.
To simplify the notation, we denote the parameter changes of the GNN trained for $p$ rounds on $T$ and $q$ rounds on $S$ as $\Theta_S(\theta_{0},q)=\sum^q_{i=0}\bigtriangledown_{\theta}L_{S}(f_{\theta_{0+i}})$ and $\Theta_T(\theta_{0},p)=\sum^p_{i=0}\bigtriangledown_{\theta}L_{T}(f_{\theta_{0+i}})$, respectively. 
Then, the initialization error at stage $n$ is:
\begin{equation}\label{eqn-7}
I_n = I(\theta^*_{n,0},\epsilon_{n-1})  = \Theta_S(\theta^{*}_{n,0}+\epsilon_{n-1},q)-\Theta_S(\theta^{*}_{n,0},q),
\end{equation}
\end{definition}
\begin{definition}
\textit{Matching error $\delta$} refers to differences at the endpoints of the same stage in the training trajectories of GNNs trained on $\mathcal{T}$ and $\mathcal{S}$ during the \textit{condensation} phase:
\begin{equation}\label{eqn-8}
\delta_{n+1}=\Theta_S(\theta^*_{n,0},q)-\Theta_T(\theta^*_{n,0},p).
\end{equation}
\end{definition}
The following theorem elucidates the relation among errors:
\begin{theorem}
\label{thm:bigtheorem}
During the \textit{evaluation} phase, the accumulated error at any stage is determined by its initial value and the sum of matching error and initialization error starting from the second stage.
\begin{equation}\label{eqn-9}
\epsilon_{n+1}  =\sum^{n}_{i=1}I({\theta^*_{i,0},\epsilon_{i-1}})+ \sum^{n}_{i=0}\delta_{i+1} +\epsilon_{0}.
\end{equation}
\end{theorem}
The proof for the above theorem can be found in Appendix~\ref{section: details_th}.
In the previous condensation method, only 
a few stages of the expert trajectory 
are selected to optimize the condensed graph. Assuming the sum of matching errors is optimized to \(\mu\) in the \textit{condensation} phase, the optimized accumulated error can be formulated as:
\begin{equation}\label{eqn-10}
\epsilon^{*}_{n+1} = \sum^{n}_{i=1}I({\theta^*_{i,0},\epsilon_{i-1}}) + \mu +\epsilon_{0}.
\end{equation}
\begin{corollary}
The proposed strategy can optimize the accumulated error in both the \textit{buffer} and \textit{condensation} phases.
\end{corollary} 
\vspace{-5.0pt}

\begin{proof}
According to~\cite{du2023minimizing}, flatter training trajectories reduce initialization error and can be derived from the following equation: 
\begin{equation}
\begin{split}
\theta^{*}_{n, i} &= \arg\min_{\theta^{*}_{n,i}} ||I(\theta^{*}_{n,i}, \epsilon_{n-1})||^2 \\
&\approx \arg\min_{\theta^{*}_{n,i}} \{ \mathcal{L}_{M}(f\left(\theta^{*}_{n,i}\right) + \alpha S(\theta^{*}_{n,i})\}
\end{split}
\end{equation}
where $\alpha$ as the coefficient that balances the robustness of $\theta^*$ to the perturbation, and $S(\theta)$ as the sharpness of the loss landscape. CL has been demonstrated in smoothing the loss landscape~\cite{sinha2020curriculum,  zhang2021curriculum}. 
Since we employ CL in the \textit{buffer} phase, we reduce $S(\theta)$ efficiently, thereby reducing accumulated error.

Moreover, employing expanding window matching to determine the matching range can involve more stages in the expert trajectories as matching targets.
This enables the direct optimization of  \(\delta_n\), thereby reducing the sum of matching errors \(\mu^{'}\). 
When conducting expanding window matching, multiple simulations of the \textit{evaluation} phase can be involved in the \textit{condensation} phase, \textit{i.e.}, training GNNs on $\mathcal{S}$ and $\mathcal{T}$ starting from \(\theta^{*}_{0,0}\), then minimize the matching error of this stage. 
This allows for the effective optimization of \(\epsilon_{0}\) and \(I({\theta^*_{1,0},\epsilon_{0}})\) in the \textit{condensation} stage as well.
\begin{equation}\label{eqn-11}
\epsilon^{*'}_{n+1}=\sum^{n}_{i=1}I^{'}({\theta^*_{i,0},\epsilon_{i-1}})+ \mu' +\epsilon^{'}_{0}<\epsilon^{*}_{n+1}.
\end{equation} 
Where $\epsilon^{'}$, \(I'\), $\mu^{'}$ are the reduced $\epsilon$, \(I\), $\mu$ respectively.
The above corollary suggests that using CL in the \textit{buffer} phase and expanding window matching in the \textit{condensation} phase can effectively reduce the accumulated error during the \textit{evaluation} phase.
\end{proof}
\vspace{-5.0pt}

\begin{algorithm}[!ht]
\caption{GEOM for condensing graph.}
\label{alg:GEOM1}
\begin{algorithmic}[1]
\State \textbf{Input:} Original graph dataset $\mathcal{T}$.

\State {\algorithmicrequire} 
$\{\tau_{p}\}$: A set of expert trajectories obtained by training ${\mathrm{GNN}}_{\mathcal{T}}$ on $\mathcal{T}$ with a curriculum learning schema.
$p$: numbers of the training steps of ${\mathrm{GNN}}_{\mathcal{S}}$
$q$: numbers of checkpoints between the start and target parameters.
${U}, {U^{'}}$: two upper bounds to determine the matching range. Initialized condensed graph $\mathcal{S}$.

\For{$k = 0, \ldots, K - 1$}
  \State Randomly sample a expert trajectory $\tau_p \sim \{\tau_{p}\}$ 
  \State Randomly sample $\theta_{t}^{*}$ and $\theta_{t+p}^{*}$, 
 where $0<t\leq{U}$
  \State Initialize $\widetilde{\theta}_{t},  
   \widetilde{\theta}_{t} = \theta_{t}^{*}$
  \For{$i = 0, \ldots, q - 1$}
    \State training GNNs on $\mathcal{S}$ and update $\widetilde {\theta}_{t+i}$ 
    \Statex\hspace{1.05cm}through Eq.~\ref{eqn-3}
  \EndFor
  \State Update condensed graph $\mathcal{S}$ through Eq.~\ref{eqn-13}
  \State {\algorithmicif} \hspace{0.1cm} ${U}< {U^{'}}$ {\algorithmicthen} 
  \State \hspace{0.4cm} ${U} = {U}+1$
\EndFor
\State \textbf{Output:} Condensed graph dataset $\mathcal{S}$.
\end{algorithmic}
\end{algorithm}

\definecolor{highlightcolor}{RGB}{198, 239, 252}  
\sethlcolor{highlightcolor}

\begin{table*}[ht]
  \centering
  \caption{Performance comparison to baselines in the node classification tasks. We achieve\textit{the highest results in most cases
  on node classification and lossless results on all datasets.} We report test accuracy (\%) on Citeseer, Cora, Ogbn-arxiv, Flickr, and Reddit. \textbf{Bold entries} are best results, \hl{highlight} marks the lossless results. Some experiments appear out of memory (oom).}
  \label{tab:nc}
  \scriptsize
   \resizebox{1.0\textwidth}{!}{\begin{tabular}{cc|ccccccc|c|c}
    \toprule
    Dataset & Ratio ($r$) & Random & Herding & K-Center & DC-Graph & GCond & GCond-X & SFGC & GEOM & Whole Dataset\\
    \midrule
    \multirow{3}{*}{Citeseer} & 0.90\% & 54.4$_{\pm 4.4}$ & 57.1$_{\pm 1.5}$ & 52.4$_{\pm 2.8}$ & 66.8$_{\pm 1.5}$ & 70.5$_{\pm 1.2}$ & 71.4$_{\pm 0.8}$ & 71.4$_{\pm 0.5}$ & 
    
    \textbf{\hl{73.0$_{\pm 0.5}$}} & \multirow{3}{*}{71.7$_{\pm 0.1}$} \\
    & 1.80\% & 64.2$_{\pm 1.7}$ & 66.7$_{\pm 1.0}$ & 64.3$_{\pm 1.0}$ &  59.0$_{\pm 0.5}$ & 70.6$_{\pm 0.9}$ & 69.8$_{\pm 1.1}$ & \hl{72.4$_{\pm 0.4}$} & 
    
    \textbf{\hl{74.3$_{\pm 0.1}$}} \\
     & 3.60\% & 69.1$_{\pm 0.1}$ & 69.0$_{\pm 0.1}$ & 69.1$_{\pm 0.1}$ & 66.3$_{\pm 1.5}$  & 69.8$_{\pm 1.4}$ & 69.4$_{\pm 1.4}$ & 70.6$_{\pm 0.7}$ & 
     
     \textbf{\hl{73.3$_{\pm 0.4}$}} \\
    \midrule
    \multirow{3}{*}{Cora} & 1.30\%  & 63.6$_{\pm 3.7}$ & 67.0$_{\pm 1.3}$ & 64.0$_{\pm 2.3}$ & 67.3$_{\pm 1.9}$ &  79.8$_{\pm 1.3}$ & 75.9$_{\pm 1.2}$ & 80.1$_{\pm 0.4}$ & 
    
    \textbf{\hl{82.5$_{\pm 0.4}$}} & \multirow{3}{*}{81.2$_{\pm 0.2}$} \\
    & 2.60\% &72.8$_{\pm 1.1}$ & 73.4$_{\pm 1.0}$ & 73.2$_{\pm 1.2}$  & 67.6$_{\pm 3.5}$ &  80.1$_{\pm 0.6}$ & 75.7$_{\pm 0.9}$ & \hl{81.7$_{\pm 0.5}$} & 
    
    \textbf{\hl{83.6$_{\pm 0.3}$}} \\
     & 5.20\%  & 76.8$_{\pm 0.1}$ & 76.8$_{\pm 0.1}$ & 76.7$_{\pm 0.1}$ & 67.7$_{\pm 2.2}$ & 79.3$_{\pm 0.3}$ & 76.0$_{\pm 0.3}$ & \hl{81.6$_{\pm 0.8}$} & 
     
     \textbf{\hl{82.8$_{\pm 0.7}$}} \\
    \midrule
    \multirow{5}{*}{Ogbn-arxiv} & 0.05\% & 47.1$_{\pm 3.9}$ & 52.4$_{\pm 1.8}$ & 47.2$_{\pm 3.0}$ & 58.6$_{\pm 0.4}$ & 59.2$_{\pm 1.1}$ & 61.3$_{\pm 0.5}$ & 65.5$_{\pm 0.7}$ & 
    \textbf{65.5$_{\pm 0.6}$} & \multirow{5}{*}{71.4$_{\pm 0.1}$} \\
    
    & 0.25\%  &57.3$_{\pm 1.1}$ & 58.6$_{\pm 1.2}$ & 56.8$_{\pm 0.8}$ & 59.9$_{\pm 0.3}$ &63.2$_{\pm 0.3}$ & 64.2$_{\pm 0.4}$ & 66.1$_{\pm 0.4}$ & 
    \textbf{68.8$_{\pm 0.2}$} \\
    
     & 0.50\%  &60.0$_{\pm 0.9}$ & 60.4$_{\pm 0.8}$ & 60.3$_{\pm 0.4}$ & 59.5$_{\pm 0.3}$ &  64.0$_{\pm 1.4}$ & 63.1$_{\pm 0.5}$ & 66.8$_{\pm 0.4}$ & 
     \textbf{69.6$_{\pm 0.2}$} \\

    & 2.50\%  &64.1$_{\pm 0.7}$ & 64.3$_{\pm 0.8}$ & 64.1$_{\pm 0.5}$ & 61.3$_{\pm 0.3}$ &  66.3$_{\pm 1.1}$ & 66.1$_{\pm 0.3}$ & 68.3$_{\pm 0.3}$ & 
     \textbf{71.0$_{\pm 0.1}$} \\

    & 5.00\%  &66.0$_{\pm 0.6}$ & 66.1$_{\pm 0.4}$ & 66.2$_{\pm 0.3}$ & 66.7$_{\pm 0.3}$ &  oom & 66.9$_{\pm 0.4}$ & 69.4$_{\pm 0.3}$ & 
     \textbf{\hl{71.4$_{\pm 0.1}$}} \\
     
      \midrule
    \multirow{3}{*}{Flickr} & 0.10\% &41.8$_{\pm 2.0}$ & 42.5$_{\pm 1.8}$ & 42.0$_{\pm 0.7}$  & 46.3$_{\pm 0.2}$ & 46.5$_{\pm 0.4}$ & 45.9$_{\pm 0.1}$ & 46.6$_{\pm 0.2}$ & 
    
    \textbf{47.1$_{\pm 0.1}$} & \multirow{3}{*}{47.2$_{\pm 0.1}$} \\
    & 0.50\%  & 44.0$_{\pm 0.4}$ & 43.9$_{\pm 0.9}$ & 43.2$_{\pm 0.1}$ & 45.9$_{\pm 0.1}$ &  \textbf{47.1$_{\pm 0.1}$} & 45.0$_{\pm 0.2}$ & 47.0$_{\pm 0.1}$ & 
    
    47.0$_{\pm 0.2}$ \\
     & 1.00\% & 44.6$_{\pm 0.2}$ & 44.4$_{\pm 0.6}$ & 44.1$_{\pm 0.4}$ & 44.6$_{\pm 0.1}$ &  47.1$_{\pm 0.1}$ & 45.0$_{\pm 0.2}$ & 47.1$_{\pm 0.1}$ & 
     
     \textbf{\hl{47.3$_{\pm 0.3}$}} \\
      \midrule

    \multirow{5}{*}{Reddit} & 0.01\%  &46.1$_{\pm 4.4}$ & 53.1$_{\pm 2.5}$ & 46.6$_{\pm 2.3}$ & 88.2$_{\pm 0.2}$ & 88.0$_{\pm 1.8}$ & 88.4$_{\pm 0.4}$ & 
    
    89.7$_{\pm 0.2}$ & \textbf{91.1$_{\pm 0.4}$} \\
    & 0.10\% &58.0$_{\pm 2.2}$ & 62.7$_{\pm 1.0}$ & 53.0$_{\pm 3.3}$  & 89.5$_{\pm 0.1}$ &  89.6$_{\pm 0.7}$ & 89.3$_{\pm 0.1}$ & 90.0$_{\pm 0.3}$ & 
    
    \textbf{91.4$_{\pm 0.2}$} &\multirow{3}{*}{93.9$_{\pm 0.0}$}\\
     & 0.20\%  &66.3$_{\pm 1.9}$ & 71.0$_{\pm 1.6}$ & 58.5$_{\pm 2.1}$ & 90.5$_{\pm 1.2}$ &  90.1$_{\pm 0.5}$ & 88.8$_{\pm 0.4}$  & 90.3$_{\pm 0.3}$ & 
     
     \textbf{91.5$_{\pm 0.4}$} \\

     & 3.00\%  &78.4$_{\pm 1.3}$ & 81.3$_{\pm 1.1}$ & 82.2$_{\pm 1.4}$ & 90.8$_{\pm 0.9}$ &  oom & 89.2$_{\pm 0.2}$  & 91.0$_{\pm 0.3}$ & 
     
     \textbf{93.7$_{\pm 0.1}$} \\

          & 5.00\%  &83.6$_{\pm 1.1}$ & 88.1$_{\pm 0.8}$ & 88.3$_{\pm 1.2}$ & 91.5$_{\pm 0.7}$ &  oom & 88.9$_{\pm 0.3}$  & 91.9$_{\pm 0.2}$ & 
     
     \textbf{\hl{93.9$_{\pm 0.1}$}} \\

    \bottomrule
    \end{tabular}}%
\end{table*}

\subsection{Knowledge Embedding Extractor}
We acquire expert trajectories in \textit{buffer} phase, with these checkpoints solely employed for trajectory matching in \textit{condensation} phase.
However, utilizing checkpoints from another perspective has not yet been explored.
Given that these checkpoints contain well-trained model parameters, which retain extensive information from the original dataset~\cite{lu2023can}.
Therefore, we try to transfer such knowledge to the condensed graph to make it more informative.

Inspired by network distillation~\cite{hinton2015distilling, zhang2021graph}, which transfers knowledge from a large model to a smaller one, we propose Knowledge Embedding Extraction (KEE), 
aiming to transfer the knowledge about the original dataset from well-trained GNNs into the condensed graph.
Specifically, we first assign soft labels to the condensed graph, which are generated by well-trained GNNs with parameters chosen from the tails of expert trajectories. 
During the optimization of $\mathcal{S}$, we feed the condensed graph with soft labels into well-trained GNNs and calculate the following loss term:
\begin{equation}\label{eqn-12}
\mathcal{L}_{E}=\mathcal{D}_{KL}\left(f\left(\theta_{T}^{*};\mathcal{S}\right)\parallel\widetilde{\mathcal{Y}}\right),
\end{equation}
Where ${\mathcal D}_{KL}(\cdot||\cdot)$ represents the Kullback-Leibler (KL) divergence, $\widetilde{\mathcal{Y}}$ denotes the soft labels. By incorporating such a matching loss, we further uncovered information embedded in the expert trajectories from a unique perspective. 


\subsection{Final Objective and Algorithm}
To sum up, the total optimization objective of GEOM is:
\begin{equation}\label{eqn-13}
\min_{S} \mathrm{E}_{\theta_{t}^{*}\sim P_{\theta_{\tau}}} \left[ \mathcal{L} \left( (\theta_{t}^{*}|^{t+p}_{t}, \widetilde{\theta}_{t}|^{t+q}_{t}),(\mathcal{S},\widetilde{\mathcal{Y}})\right) \right], \mathrm{where}
\end{equation}
\begin{equation}\label{eqn-14}
\begin{split}
\mathcal{L} &= \mathcal{L}_{M} + \alpha\mathcal{L}_{E} \\
&=\frac { \left \| \widetilde{\theta}_{t+q}-\theta_{t+p}^{*}\right\|_2^2}{\left\|\widetilde{\theta}_{t}-\theta_{t+p}^{*}\right\|_2^2}, + \alpha\mathcal{D}_{KL}\left(f\left(\theta_{T}^{*};\mathcal{S}\right)\parallel\widetilde{\mathcal{Y}}\right).
\end{split}
\vspace{-10.0pt}
\end{equation}
The pipeline of the proposed GEOM is detailed in Alg.~\ref{alg:GEOM1}.


\section{Experiments} \label{experiments}

\definecolor{highlightcolor}{RGB}{198, 239, 252}  
\sethlcolor{highlightcolor}

\begin{table*}[ht]
\centering
\tiny
\renewcommand\arraystretch{0.8}
\setlength{\extrarowheight}{1.5pt} 
\caption{Performance across different GNN architectures. \colorbox{orange!30}{\makebox(14,6){Avg.}} and \colorbox{red!30}{\makebox(15,6){Std.}}: the average performance and the standard deviation of the results of all architectures, \colorbox{cyan!20}{\makebox(20,6){$\Delta$(\%)}} denotes the improvements upon the DC-Graph. \colorbox{gray!20}{GCN} indicates that the synthetic graph is condensed with GCN. \textbf{Bold entries} are best results.}
\hspace{3pt}
\label{tab:cross}
\resizebox{0.8\textwidth}{!}{%
\begin{tabular}{cccccccccccccc}
\cline{1-12}
                          &         & \multicolumn{7}{c}{Architectures}                & \multicolumn{3}{c}{Statistics}                                                                                          \\ \cline{3-12}
\multirow{-2}{*}{Datasets} & \multirow{-2}{*}{Methods} & MLP  & GAT  & APPNP & Cheby & \cellcolor{gray!40}GCN & SAGE & SGC & \cellcolor[HTML]{FFE6CC}Avg.          & \cellcolor[HTML]{FFB2B2}Std.          & \cellcolor[HTML]{C6EFFC}$\Delta$(\%)  \\ \cline{1-9}

                                                      & DC-Graph & 66.2 & - &66.4 &64.9 &\cellcolor{gray!40}66.2 &65.9 &69.6 &  \cellcolor[HTML]{FFE6CC}66.5          & \cellcolor[HTML]{FFB2B2}1.5         & \cellcolor[HTML]{C6EFFC}-               \\

                          & GCond & 63.9 &55.4 &69.6 &68.3 &\cellcolor{gray!40}70.5 &66.2 &70.3 & \cellcolor[HTML]{FFE6CC}66.3         & \cellcolor[HTML]{FFB2B2}5.0         & $\downarrow$ \cellcolor[HTML]{C6EFFC}0.2       \\
                           & SFGC & {71.3} &  {72.1} &70.5 & {71.8}& \cellcolor{gray!40}71.6 & {71.7}&  {71.8}&  \cellcolor[HTML]{FFE6CC}71.5          & \cellcolor[HTML]{FFB2B2}0.5          & $\uparrow$ \cellcolor[HTML]{C6EFFC}5.0         \\

\multirow{-5}{*}{\centering Citeseer}
& GEOM &  {74.2} &  {74.2} &  {74.0} &  {74.1}  &\cellcolor{gray!40} {74.3} &  {74.1} &  {74.3} & \cellcolor[HTML]{FFE6CC}\textbf{74.2} & \cellcolor[HTML]{FFB2B2}0.1 & $\uparrow$ \cellcolor[HTML]{C6EFFC}\textbf{7.7}   \\ \cline{1-9}
\multirow{-5}{*}{\centering ($r$ = 1.80\%)}

                                                          & DC-Graph & 67.2 & - &67.1 &67.7 &\cellcolor{gray!40}67.9 &66.2 &72.8 &  \cellcolor[HTML]{FFE6CC}68.1          & \cellcolor[HTML]{FFB2B2}2.1          & \cellcolor[HTML]{C6EFFC}-               \\

                          & GCond & 73.1 &66.2 &78.5 &76.0 &\cellcolor{gray!40}80.1 &78.2 &79.3 & \cellcolor[HTML]{FFE6CC}75.9          & \cellcolor[HTML]{FFB2B2}4.5         &  $\uparrow$ \cellcolor[HTML]{C6EFFC}7.8         \\
                           & SFGC &  {81.1}& 80.8 &78.8 &79.0& \cellcolor{gray!40}81.1 & {81.9}& 79.1&  \cellcolor[HTML]{FFE6CC}80.3         & \cellcolor[HTML]{FFB2B2}1.2          & $\uparrow$ \cellcolor[HTML]{C6EFFC}12.2            \\

\multirow{-5}{*}{Cora}  & GEOM &  {83.6} & 82.7 & 82.8 & 80.7  &\cellcolor{gray!40}{83.6} &  {83.7} &  {83.1} & \cellcolor[HTML]{FFE6CC}\textbf{82.9} & \cellcolor[HTML]{FFB2B2}1.0 & $\uparrow$ \cellcolor[HTML]{C6EFFC}\textbf{14.8}  \\ \cline{1-9}
\multirow{-5}{*}{\centering ($r$ = 2.60\%)}
                          & DC-Graph & 59.9& -& 60.0& 55.7 &\cellcolor{gray!40}59.8 &60.0 & 60.4 & \cellcolor[HTML]{FFE6CC}59.3          & \cellcolor[HTML]{FFB2B2}1.6           & \cellcolor[HTML]{C6EFFC}-               \\

                          & GCond & 62.2& 60.0 &63.4 &54.9& \cellcolor{gray!40}63.2 &62.6& 63.7&  \cellcolor[HTML]{FFE6CC}61.4          & \cellcolor[HTML]{FFB2B2}2.9          &  $\uparrow$ \cellcolor[HTML]{C6EFFC}2.1              \\
                           & SFGC & 65.1& 65.7 &63.9 &60.7& \cellcolor{gray!40}65.1 &64.8& 64.8&  \cellcolor[HTML]{FFE6CC}64.3          & \cellcolor[HTML]{FFB2B2}1.6         & $\uparrow$ \cellcolor[HTML]{C6EFFC}5.0             \\
           
\multirow{-5}{*}{Ogbn-arxiv}   & GEOM &  {68.8} & 66.4 &65.8 & 62.5 &\cellcolor{gray!40}68.8 & 68.9 & 66.4  & \cellcolor[HTML]{FFE6CC}\textbf{66.8} & \cellcolor[HTML]{FFB2B2}2.1  & $\uparrow$ \cellcolor[HTML]{C6EFFC}\textbf{7.5}   \\ \cline{1-9}
\multirow{-5}{*}{\centering ($r$ = 0.25\%)}

                           & DC-Graph & 61.2 & - &61.4 &58.3& \cellcolor{gray!40}61.1 &60.9& 61.3&  \cellcolor[HTML]{FFE6CC}60.7          & \cellcolor[HTML]{FFB2B2}1.1         &  \cellcolor[HTML]{C6EFFC}-             \\
                           & SFGC &  {69.4}& 69.6 &68.7 &64.7& \cellcolor{gray!40}69.4 &69.4& 69.1&  \cellcolor[HTML]{FFE6CC}68.5          & \cellcolor[HTML]{FFB2B2}1.6         &  \cellcolor[HTML]{C6EFFC}$\uparrow$7.8            \\
           
\multirow{-3.7}{*}{Ogbn-arxiv}   & GEOM &  {71.2} & 70.0 &69.1 & 64.5 &\cellcolor{gray!40}{71.4} & 71.1 & 69.6  & \cellcolor[HTML]{FFE6CC}\textbf{69.6} & \cellcolor[HTML]{FFB2B2}2.2  & $\uparrow$ \cellcolor[HTML]{C6EFFC}\textbf{8.9}   \\ \cline{1-9}
\multirow{-3.5}{*}{\centering ($r$ = 5.00\%)}

                          & DC-Graph &43.1 & -& 45.7 &43.8 &\cellcolor{gray!40}45.9 &45.8 &45.6  & \cellcolor[HTML]{FFE6CC}45.0        & \cellcolor[HTML]{FFB2B2}1.1         & \cellcolor[HTML]{C6EFFC}-           \\
                           & GCond & 44.8 & 40.1 & 45.9 & 42.8 & \cellcolor{gray!40}47.1& 46.2& 46.1 & \cellcolor[HTML]{FFE6CC}44.7         & \cellcolor[HTML]{FFB2B2}2.3           & $\downarrow$ \cellcolor[HTML]{C6EFFC}0.3              \\
                           & SFGC & 47.1 &  {45.3} & 40.7 & 45.4 &\cellcolor{gray!40}47.1&  {47.0}& 42.5 & \cellcolor[HTML]{FFE6CC}45.0        & \cellcolor[HTML]{FFB2B2}2.3           & \cellcolor[HTML]{C6EFFC}-             \\
\multirow{-5}{*}{Flickr}     & GEOM & 47.0 & 42.1 & 46.6 & 45.3 & \cellcolor{gray!40}{47.3} & {47.1} & {46.3} & \cellcolor[HTML]{FFE6CC}\textbf{46.0} & \cellcolor[HTML]{FFB2B2}1.7  & $\uparrow$ \cellcolor[HTML]{C6EFFC}\textbf{1.0}   \\ \cline{1-9}
\multirow{-5}{*}{\centering ($r$ = 0.50\%)}
                           & DC-Graph & 50.3 & - & 81.2 & 77.5 & \cellcolor{gray!40}89.5 & 89.7 & 90.5 & \cellcolor[HTML]{FFE6CC}79.8         & \cellcolor[HTML]{FFB2B2}14.0          & \cellcolor[HTML]{C6EFFC}-              \\
                           & GCond & 42.5 & 60.2 & 87.8 & 75.5 & \cellcolor{gray!40}89.4 & 89.1 & 89.6 & \cellcolor[HTML]{FFE6CC}76.3          & \cellcolor[HTML]{FFB2B2}17.1          &  $\downarrow$ \cellcolor[HTML]{C6EFFC}3.5             \\
                          & SFGC & 89.5 & 87.1 & 88.3 & 82.8 & \cellcolor{gray!40}89.7 & 90.3 & 89.5 & \cellcolor[HTML]{FFE6CC}88.2          & \cellcolor[HTML]{FFB2B2}2.4           & $\uparrow$ \cellcolor[HTML]{C6EFFC}8.4             \\
\multirow{-5}{*}{Reddit} & GEOM & 91.4 & 90.0 & 87.9 & 82.7 & \cellcolor{gray!40}91.4 & 91.4 & 89.3 & \cellcolor[HTML]{FFE6CC}\textbf{89.2}  & \cellcolor[HTML]{FFB2B2}2.9  & \cellcolor[HTML]{C6EFFC}$\uparrow$ \textbf{9.4}  \\ \cline{1-12}
\multirow{-5}{*}{\centering ($r$ = 0.10\%)}

                           & DC-Graph & 52.3 & - &84.1 &81.2& \cellcolor{gray!40}90.3 &90.6& 90.9&  \cellcolor[HTML]{FFE6CC}81.6          & \cellcolor[HTML]{FFB2B2}13.6         &  \cellcolor[HTML]{C6EFFC}-             \\

                           & SFGC & 91.6 & 90.3 &91.1 &85.3& \cellcolor{gray!40}91.9 &91.6& 90.9&  \cellcolor[HTML]{FFE6CC}90.4          & \cellcolor[HTML]{FFB2B2}2.1         &  \cellcolor[HTML]{C6EFFC}$\uparrow$8.8             \\
\multirow{-3.7}{*}{Reddit} & GEOM &  {93.9} &  {93.0} & 92.4 & 88.5 & \cellcolor{gray!40}{93.9} &  {93.8} & 92.7 & \cellcolor[HTML]{FFE6CC}\textbf{92.6}  & \cellcolor[HTML]{FFB2B2}1.8  & \cellcolor[HTML]{C6EFFC}$\uparrow$ \textbf{11.0}   \\ \cline{1-12}
\multirow{-3.5}{*}{\centering ($r$ = 5.00\%)}
\end{tabular}%
}
\vspace{-2.5em}
\end{table*}

\subsection{Setup}
\textbf{Datasets \& architectures.} 
We conduct experiments on three transductive datasets, \textit{i.e.}, Cora, Citeseer~\cite{kipf2016semi} and Ogbn-arxiv~\cite{hu2020open}, and two inductive datasets, \textit{i.e}, Flickr~\cite{zeng2020graphsaint} and Reddit~\cite{hamilton2017inductive}. For all five datasets, we use the public splits and setups. More details of each dataset can be found in Appendix~\ref{section: details_data}. We select APPNP~\cite{gasteiger2018predict}, GCN~\cite{kipf2016semi}, SGC~\cite{wu2019simplifying}, GraphSAGE~\cite{hamilton2017inductive}. Cheby~\cite{defferrard2016convolutional} and GAT~\cite{velivckovic2017graph}, as well as a standard MLP for cross-architecture experiments.

\textbf{Baselines.}
We compare our method to seven baselines: 
1) Coreset selection methods: 
Random; Herding~\cite{welling2009herding}; K-Center~\cite{farahani2009facility, sener2017active}.
2) State-of-the-art condensation methods: 
the graph-based variant DC-Graph of vision dataset condensation~\cite{zhao2020dataset}; gradient matching graph condensation method GCond~\cite{jin2022graph}; 
GCond-X, the variant of GCond, which do not optimize the structure of the condensed graph; 
trajectory matching graph condensation method SFGC~\cite{zheng2023structure}.

\begin{figure*}[t]
\centering
    \subfigure[]
    {\includegraphics[width=0.21\textwidth, angle=0]
    {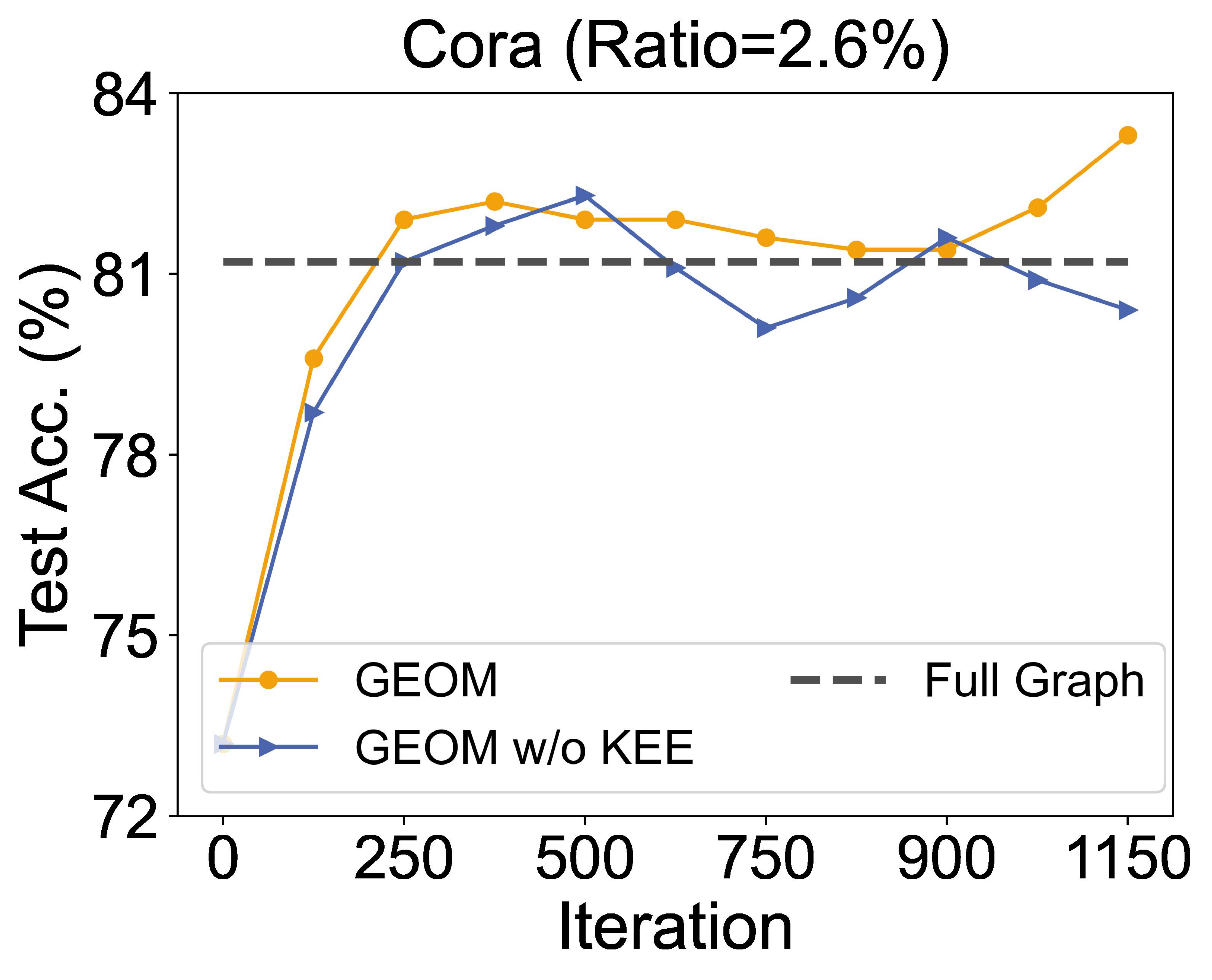}
    \label{fig:abl_cora}}
    \hspace{3.5mm}
    \subfigure[]
    {\includegraphics[width=0.21\textwidth, angle=0]
    {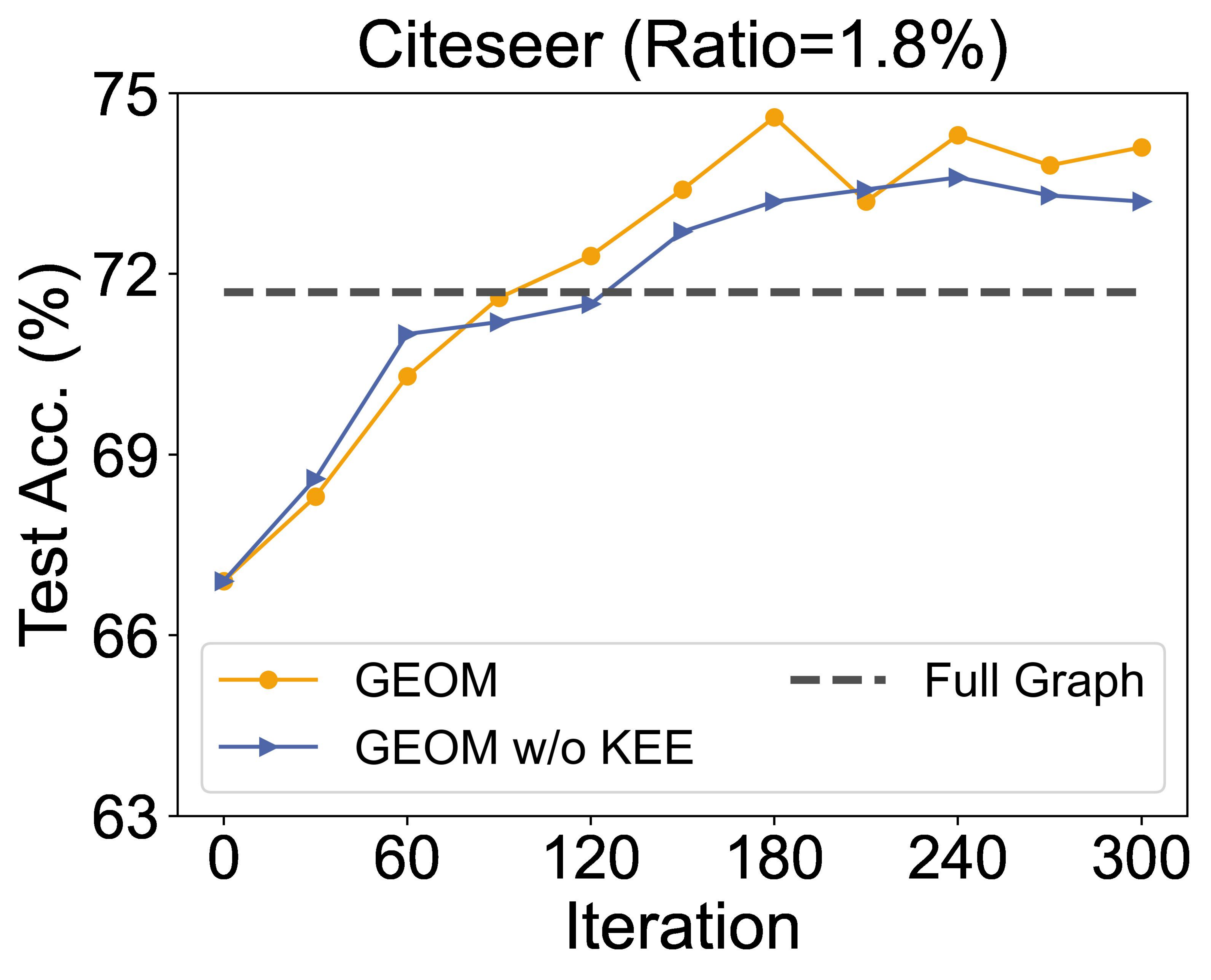}
    \label{fig:abl_red}}
    \hspace{3.5mm}
    \subfigure[]
    {\includegraphics[width=0.211\textwidth, angle=0]{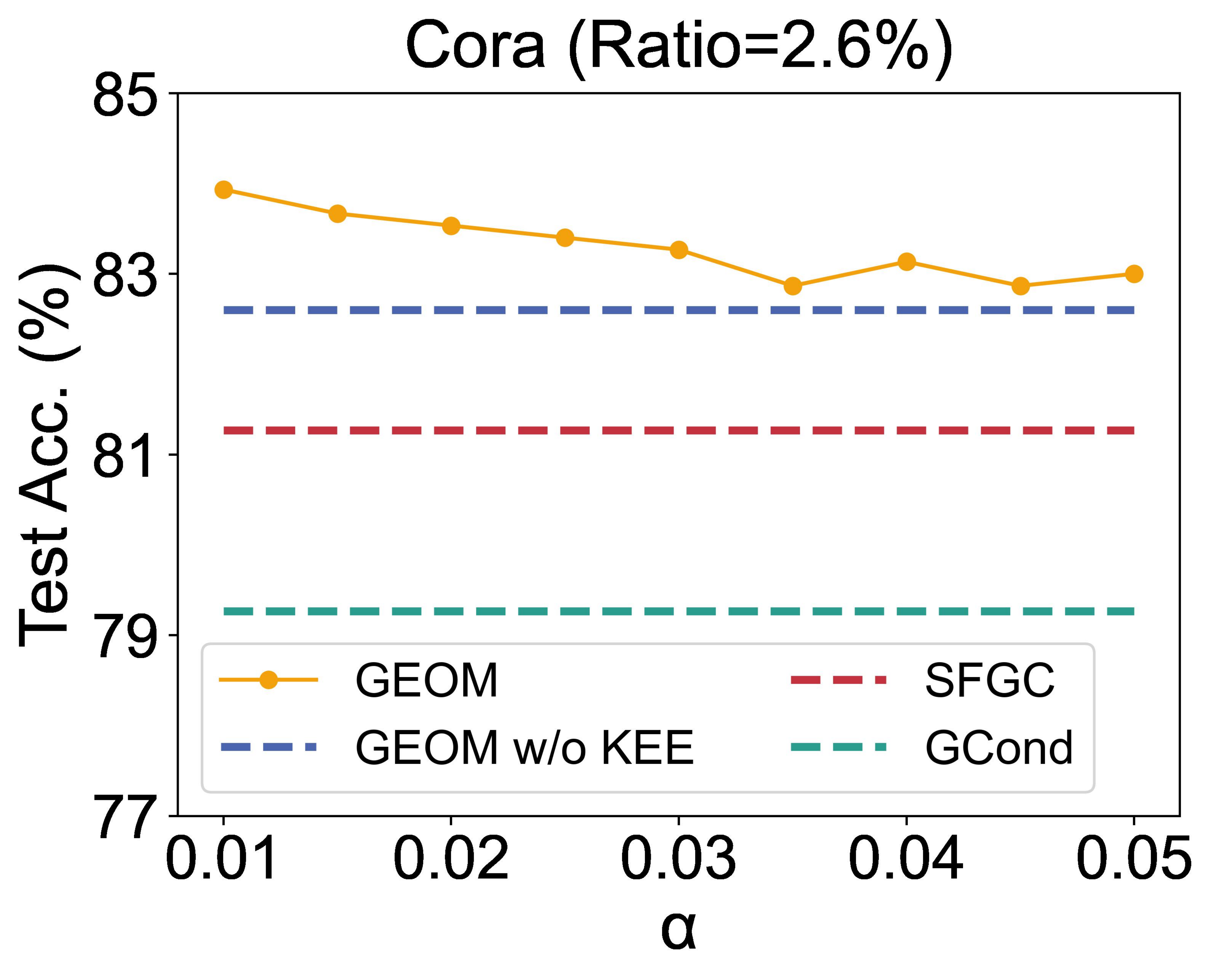}
    \label{fig:alpha_cora}}
    \hspace{3.5mm}
    \subfigure[]
    {\includegraphics[width=0.212\textwidth, angle=0]{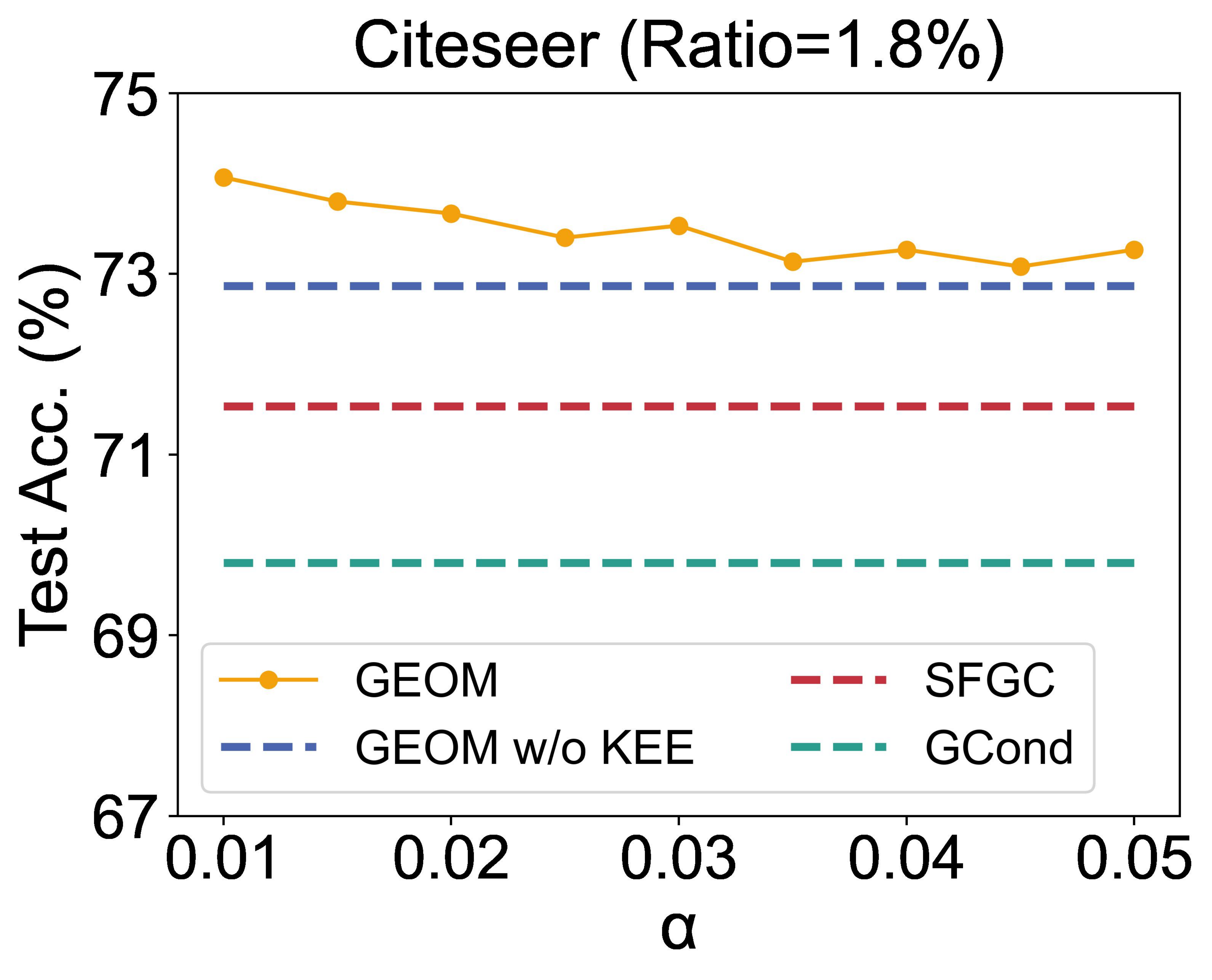}
    \label{fig:alpha_reddit}}
    \vspace{-10pt}
\caption{(a) and (b) illustrate the ablation study on whether to use the KEE. (c) and (d) illustrate the ablation of the tunable hyperparameter $\alpha$, which determines the weights of the optimization item generated by the KEE.}\label{fig:exp}
    \vspace{-8.0pt}
\end{figure*}

\textbf{Implementation \& evaluation.}
We initially employ the eight methods to synthesize condensed graphs, subsequently assessing the performance of GNNs across different datasets and architectures. 
In the \textit{condensation} phase, GNNs are both commonly-used GCN model~\cite{kipf2016semi}. 
In the \textit{evaluation} phase, we train a GNN on the condensed graph and then evaluate the GNN on the test set of the corresponding original graph dataset to get the performance. 
We report the average performance and variance on Table~\ref{tab:nc} with repeated 10 times experiments, where the GNN models used are all 2-layer GCN models with 256 hidden units. 
More hyper-parameter setting details are provided in Appendix~\ref{section: details}.

\subsection{Results}
\textbf{Node classification.}
We compare our method with the baselines across all condensation ratios on node classification, as reported in Table~\ref{tab:nc}. 
Our method achieves state-of-the-art results in 18 out of 19 experimental cases and brings non-trivial improvements up to 2.8\%. 
Notably, in all datasets, we achieve \textit{lossless} graph condensation below or at a 5\% condensation ratio for the first time. 
The results confirm that the proposed GEOM can provide more informative supervision signals from the original graph for optimizing the condensed graph in the \textit{condensation} phase, allowing us to get an optimal substitute for the original graph dataset.

\begin{figure}[ht]
    \centering
    \subfigure[Cora, $r$=5.2\%]
    {\includegraphics[width=0.4\textwidth, angle=0]
    {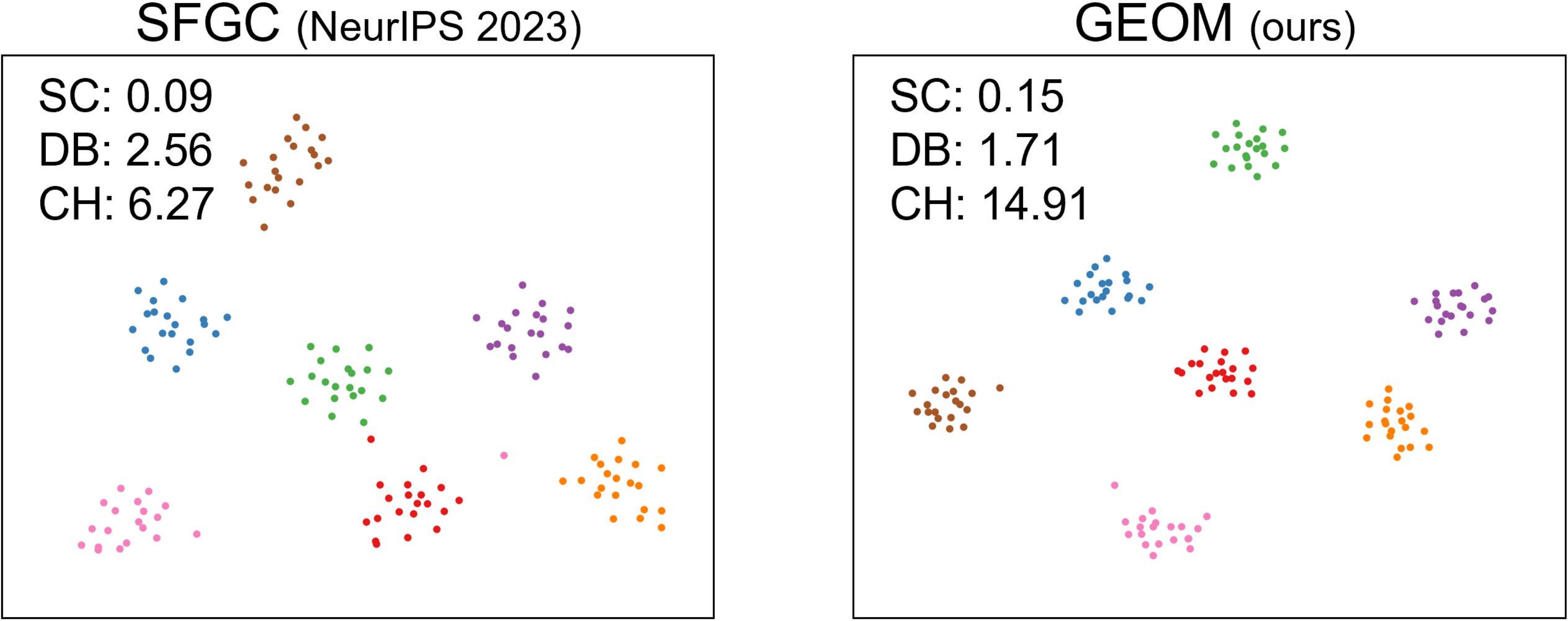}
    \label{fig:vis_cora}}
    \subfigure[Citeseer, $r$=3.6\%]
    {\includegraphics[width=0.4\textwidth, angle=0]{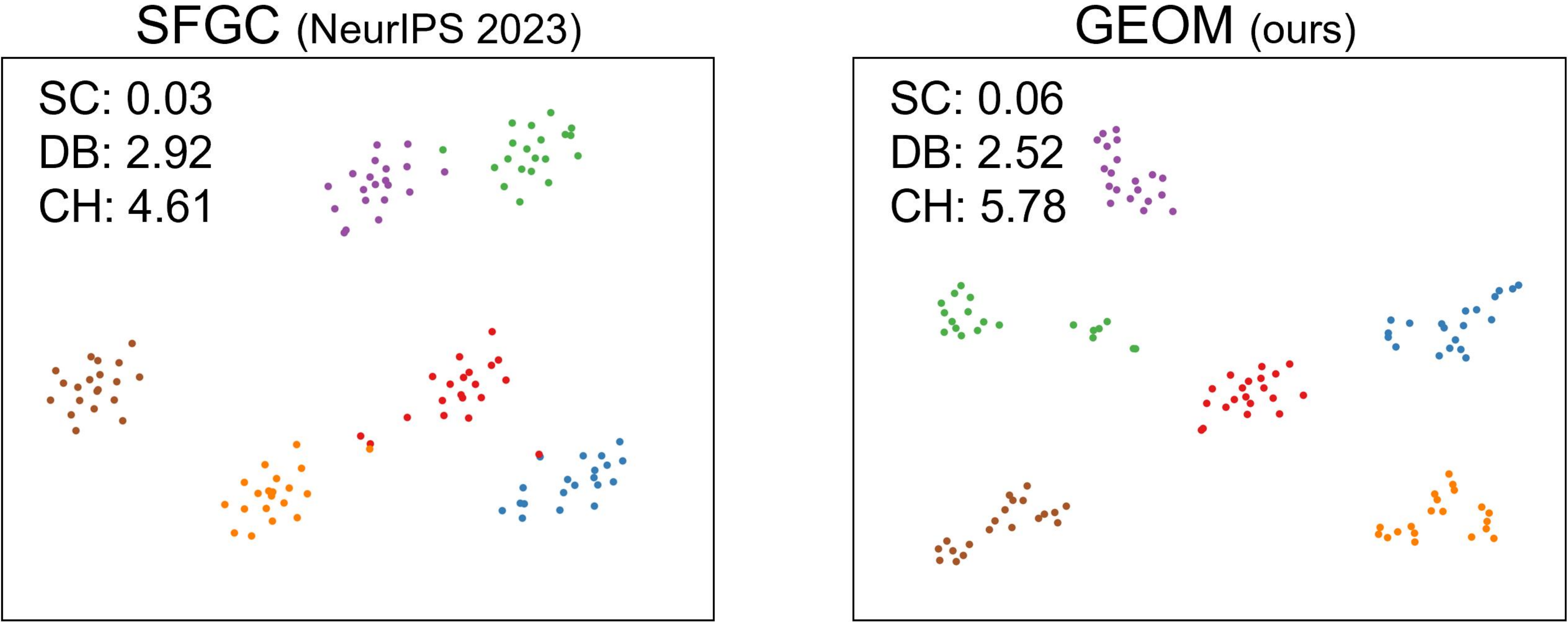}
    \label{fig:vis_citeseer}}
    \subfigure[Reddit, $r$=0.2\%]
    {\includegraphics[width=0.4\textwidth, angle=0]{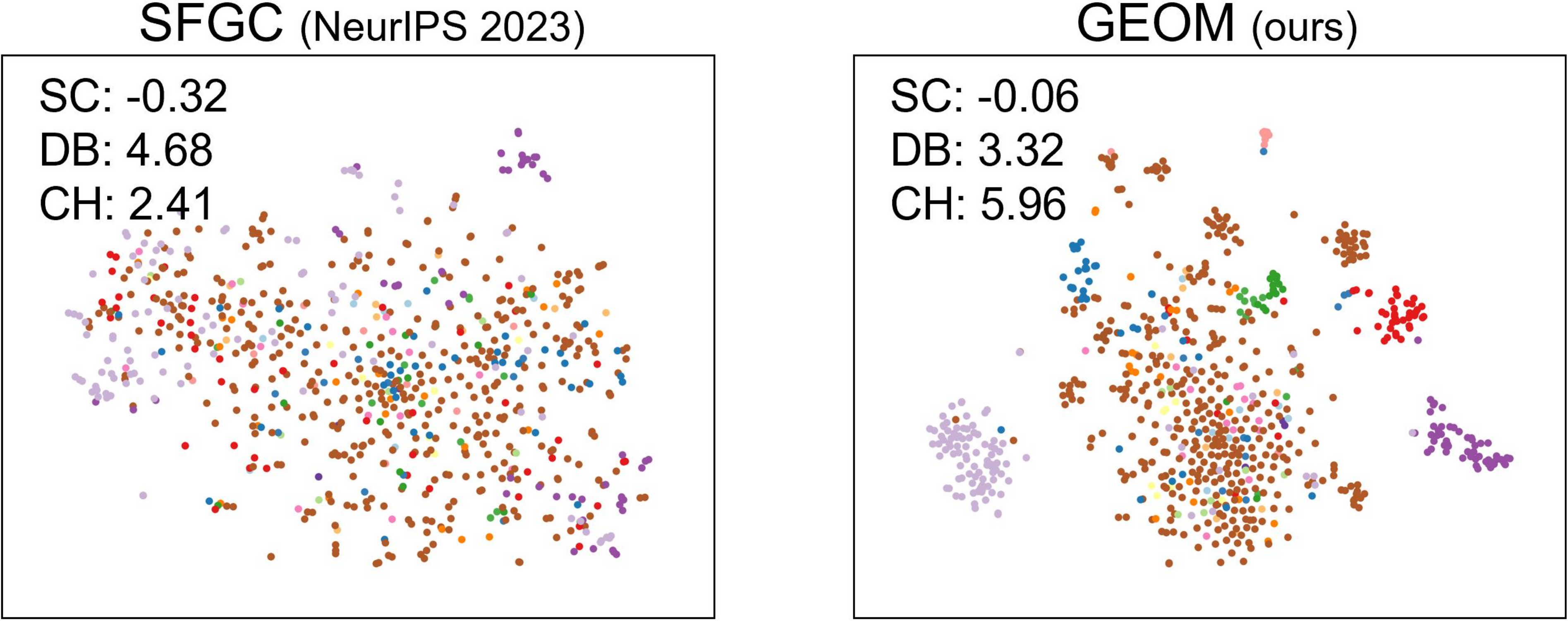} 
    \label{fig:vis_reddit}}
    \vspace{-10pt}
	\caption{T-SNE visualization on the condensed graph. Nodes of the same class are in the same color. SC$\uparrow$, DB$\downarrow$, and CH$\uparrow$ in the figure refer to the Silhouette Coefficient, Davies-Bouldin Index, and Calinski-Harabasz Index respectively. $\uparrow$ and $\downarrow$ denote the clustering pattern is better when the value is higher or lower.}\label{fig:vis_syngraph}
    \vspace{-20.0pt}
\end{figure}

\textbf{Cross-architecture generalization.}
We evaluate the test performance of our condensed graphs across different GNN architectures. 
The results are reported in Table~\ref{tab:cross},
showing that our condensed graphs do not overfit in the GNN architecture used in the \textit{buffer} and \textit{condensation} phase, they can generalize well on all other GNN architectures in our experiments.
It is noteworthy that our condensed graph can even achieve lossless performance in 20 out of 35 cases, which opens up possibilities for the widespread real-world application of graph condensation.
As our approach shows that it is possible to condense graphs without tailoring the condensation to specific GNN architectures separately.
The performance of the whole dataset across different architectures can be found in Appendix~\ref{section: details_data}. 
%


\subsection{Ablation}
\paragraph{Evaluating expanding window matching.}
We compare expanding window matching to two fixed matching (Fixed1 with its starting point fixed at 0, Fixed2 at a later stage) and a sliding window matching.
Additionally, we conduct the ablation on whether to use CL in the \textit{buffer} phase. 
The results in Table~\ref{tab:abl} show that: 
1) Utilizing expanding window matching solely in the \textit{condensation} phase can yield better results compared to other matching strategies;
2) Fixed matching can not collaborate well with the expert trajectories trained with CL;
3) The combination of using CL and expanding window matching can bring non-trivial improvements.

\begin{table}[ht]
  \centering
    \footnotesize
    \fontsize{8.5}{9}\selectfont
    \vspace{-20.0pt}
  \caption{Ablation on expanding window matching.}
   \vspace{2.0pt}
  \label{tab:abl}
  \begin{tabular}{ccccc}
    \toprule
    \multirow{2}{*}{Matching Range} & \multicolumn{2}{c}{Cora}& \multicolumn{2}{c}{Citeseer} \\ 
    \cmidrule{2-5}
     & w/o CL & CL& w/o CL & CL \\ 
    \midrule
    Fixed1 & 82.3$_{\pm 0.4}$ & 82.6$_{\pm 0.2}$ &  72.1$_{\pm 0.2}$ & 72.4$_{\pm 0.1}$  \\
    Fixed2 & 81.8$_{\pm 0.5}$ & 81.9$_{\pm 0.3}$ & 71.3$_{\pm 0.1}$ & 71.3$_{\pm 0.3}$  \\
    Sliding & 80.8$_{\pm 0.3}$ & 81.1$_{\pm 0.4}$ & 70.7$_{\pm 0.1 }$ & 71.0$_{\pm 0.2}$ \\
    Expanding & 82.9$_{\pm 0.2}$ & \textbf{83.6$_{\pm 0.3}$} & 73.2$_{\pm 0.3}$ & \textbf{74.3$_{\pm 0.1}$} \\
    \bottomrule
  \end{tabular}
  \vspace{-8.0pt}
\end{table}



\textbf{Evaluating knowledge embedding extractor.}
We first study the \textit{condensation} phase with and without KEE under the optimal $\alpha$. As illustrated in Fig.~\ref{fig:abl_cora} and~\ref{fig:abl_red}, due to the stable guidance KEE offers, it can further enhance the performance of the condensed graph when the previous optimization hits a bottleneck.
Moreover, as shown in Fig.~\ref{fig:alpha_cora} and ~\ref{fig:alpha_reddit}, simply determining the order of the hyper-parameter $\alpha$ can make KEE aid in condensation.



\subsection{Visualization}
\vspace{-0.3em}
We present the visualization results for all five datasets in Fig.~\ref{fig:vis_syngraph}, 
we can observe that our condensed graphs on Cora and Citeseer show clear clustering patterns without inter-class mixing, while that of SFGC still lack clarity in the separation between classes. 
This gap becomes more apparent in larger datasets, where the condensed graph of SFGC fails to show any significant clustering patterns.
Conversely, our condensed graph still manages to represent clear patterns between classes and clusters within the same class. 

To measure the clustering patterns more precisely, we employee three different metrics designed to assess the distinctness of the clustering pattern comprehensively: Silhouette Coefficient~\cite{rousseeuw1987silhouettes}, Davies-Bouldin Index~\cite{davies1979cluster} and Calinski-Harabasz Index~\cite{calinski1974dendrite}.
Our condensed graph is significantly more competitive across these three evaluation metrics.
This demonstrates that our matching strategy effectively captures the patterns of both easy and difficult nodes in the original graph. 
More visualization results on other ratios and datasets can be found in Appendix~\ref{appendix_vis}.

\vspace{-0.5em}
\section{Related Work}
\vspace{-0.3em}
\textbf{Dataset Distillation \& Graph Condensation.}
Dataset distillation (DD) is a technique to reduce the size and complexity of large-scale datasets for training deep neural networks~\cite{qin2023infobatch}. Methods in DD are majorly based on matching, such as matching gradients~\cite{zhao2020dataset,liu2022dataset, liu2023slimmable, zhang2023accelerating}, distribution~\cite{zhao2023dataset, wang2023dim}, feature~\cite{wang2022cafe} and training trajectories~\cite{cazenavette2022dataset, du2023minimizing, guo2023towards}, which has led to a wide application in lots of downstream tasks~\cite{masarczyk2020reducing, rosasco2021distilled, wang2021rethinking, liu2023cat, gao2024enhancing}. Following DD, Graph Condensation compresses the graph dataset by matching gradients~\cite{jin2022condensing, jin2022graph, yang2024does,zhang2024Crafting}, distribution~\cite{liu2022graph} and training trajectories~\cite{zheng2023structure}. For a thorough review, we refer the reader to a recent survey~\cite{hashemi2024comprehensive}.
However, there persists a significant performance gap to get \textit{lossless graph condensation}. 
\vspace{-0.5em}
\section{Conclusion} \label{conclusion}
In this work, we propose GEOM, a novel method for graph condensation via expanding window matching. GEOM outperforms the state-of-the-art methods across various datasets and architectures, making the first attempt toward \textit{lossless graph condensation}.

\textbf{Limitations and future work.} GEOM still relies on deriving trajectories in advance, which incurs additional computational costs for expert GNNs training. We will explore improving the efficiency of condensing graphs in the future.


\section*{Impact Statement} \label{Impact Statements}
\textbf{Ethical impacts.} There are no ethical issues in our paper, including its motivation, designs, experiments, and used data. 
The goal of the proposed GEOM is to advance the field of graph condensation.

\textbf{Expected societal implications.}
Training GNNs on real-world graph data comes with a high computational cost. 
Graph condensation achieves more efficient computation by reducing the size of the graph data. 
This helps in reducing energy consumption in computing devices, thereby reducing carbon emissions, which is highly beneficial for sustainability and environmental conservation.
\section*{Acknowledgement}
This research is supported by the National Research Foundation, Singapore under its AI Singapore Programme (AISG Award No: AISG2-PhD-2021-08-008). Yang You's research group is being sponsored by NUS startup grant (Presidential Young Professorship), Singapore MOE Tier-1 grant, ByteDance grant, ARCTIC grant, SMI grant (WBS number: A-8001104-00-00),  Alibaba grant, and Google grant for TPU usage.

\section*{Contribution Statement}
In this paper, the authors made the following contributions:
\begin{itemize}

\item Yuchen Zhang proposed GEOM and implemented it. He also designed the experiments, conducted part of the experiments, analyzed the results, designed the entire logic, plotted the figures, and wrote the majority of the manuscript.

\item Tianle Zhang conducted part of the experiments and recorded all the experimental results. He also wrote the theoretical analysis of GEOM.

\item Kai Wang designs the logic of the abstract and introduction with Yuchen, modifies the abstract and introduction sentence-by-sentence with Yuchen, improves the storytelling, and organizes the rebuttal (analyzing the questions and replying to the reviewers) with Yuchen and Tianle.

\item Ziyao Guo, Yuxuan Liang, and Xavier Bresson provided critical feedback and revised the manuscript.

\item Wei Jin and Yang You supervised the project and provided valuable feedback about the work.

\end{itemize}

\bibliography{example_paper}
\bibliographystyle{icml2024}

\newpage
\appendix
\onecolumn


\section{Dataset Details}\label{section: details_data}

\subsection{Statistics of Dataset}
The performance assessment of our method encompasses an array of datasets, comprising three transductive datasets:  Cora, Citeseer~\cite{kipf2016semi}, Ogbn-arxiv~\cite{hu2020open} and two inductive datasets: Flickr~\cite{zeng2020graphsaint} and Reddit~\cite{hamilton2017inductive}). These datasets are sourced from PyTorch Geometric~\cite{fey2019fast}, with publicly accessible splits consistently applied across all experimental setups. We first set three condensation ratios for each dataset, consistent with the setting before~\cite{zheng2023structure, jin2022graph}, and for Ogbn-arxiv and Reddit, we add comparisons with two additional larger condensation ratios. Dataset statistics are shown in Table \ref{tab:ncdatasets}
\vspace{-10pt} 
\begin{table*}[ht]
  \centering
  \caption{Dataset statistics. The first three are transductive datasets and the last two are inductive
datasets.}
\vspace{3pt} 
  \label{tab:datasets}
  \fontsize{10}{11}\selectfont
   \begin{tabular}{c@{\hspace{2em}}c@{\hspace{2em}}c@{\hspace{2em}}c@{\hspace{2em}}c@{\hspace{2em}}c}
    \toprule
    Dataset & \#Nodes & \#Edges & \#Classes & \#Features & Training/Validation/Test\\
    \midrule
    Cora & 2,708 & 5,429 & 7 & 1,433 & 140/500/1000 \\
    Citeseer & 3,327 & 4,732 & 6 & 3,703 & 120/500/1000\\
Ogbn-arxiv & 169,343 & 1,166,243 & 40 & 128 & 90,941/29,799/48,603\\
\midrule
Flickr & 89,250 & 899,756 & 7& 500 & 44,625/22312/22313\\
Reddit & 232,965 & 57,307,946 & 210 & 602 & 15,3932/23,699/55,334\\
    \bottomrule
    \end{tabular}%
  \label{tab:ncdatasets}%
\end{table*}

\subsection{Performance of Dataset}
We show the performances of various GNNs on the original graph datasets in Table \ref{tab:cross_whole}. Notably, our approach achieves lossless performance for 20 combinations across 35 combinations of five datasets and seven architectures

  \begin{table*}[ht]
      \centering
      \small
        \caption{Performances of various GNNs on original graphs. The \underline{underline} signifies that the performance of our synthetic graph is the same as or better than the original graph dataset.}
      \begin{tabular}{c|ccccccc}
       \toprule
           & MLP&GAT& APPNP& Cheby & GCN & SAGE& SGC\\
          \midrule
    Citeseer &\underline{69.1} & \underline{70.8}& \underline{71.8}& \underline{70.2} & \underline{71.7} & \underline{70.1}& \underline{71.3}\\
    \midrule
    Cora & \underline{76.9} & 83.1 & 83.1& 81.4 & \underline{81.2} & \underline{81.2}& \underline{81.4}\\
    \midrule
    Ogbn-arxiv  & \underline{67.8} & 71.5& 71.2&71.4& \underline{71.4} &71.5& 71.4 \\
        \midrule
    Flickr & 47.6 & 44.3 &  47.3 & 47.0  &  \underline{47.1} & \underline{46.1} & \underline{46.2}\\
    \midrule
    Reddit &\underline{92.6} &  \underline{91.0} & 94.3&93.1 &  \underline{93.9} &\underline{93.0}& 93.5\\

            \bottomrule
      \end{tabular}
    
      \label{tab:cross_whole}
  \end{table*}
\section{Datails of the Training Scheduler}\label{training sch}
After assessing node difficulty, we implement a curriculum-based approach to train a GNN model. Following CLNode~\cite{wei2023clnode}, We introduce a continuous training scheduler that gradually increases the difficulty level in the curriculum. Specifically, we organize the training set by the ascending node difficulty. Then, using a pacing function $h(t)$ to map each epoch to a certain value $\lambda_t$, where $0< \lambda_t\leq1$, indicating the proportion of the training nodes selected for the training subset at epoch $\zeta$. $\lambda$ represents initial proportion of the available nodes, $\zeta$ is the epoch when $h(t)$ attains the value of 1. The pacing functions are as follows:
\begin{itemize}
\item linear:
\begin{equation}\label{eqn-15}
h(t)=\min(1,\lambda+(1-\lambda)\frac{t}{\zeta}).
\end{equation}
\item root:
\begin{equation}\label{eqn-16}
h(t)=\min(1,\sqrt{\lambda^2+(1-\lambda^2)\frac t\zeta}).
\end{equation}
\item geometrics:
\begin{equation}\label{eqn-17}
h(t)=\min(1,2^{log_2\lambda-log_2\lambda\frac t\zeta}).
\end{equation}
\end{itemize}
Furthermore, we do not halt the training as soon as $t$ equals $\zeta$, since at this point, 
the knowledge of difficult nodes might not be fully embedded into the expert trajectories. Therefore, we continue to train the model for an additional period to ensure that the information of these difficult nodes is also embedded into the expert trajectories.



\section{Theoretical Analysis}
\subsection{Proof of Theorem 2.4}\label{section: details_th}
\begin{theorem}
During the \textit{evaluation} phase, the accumulated error at any stage is determined by its initial value, the sum of matching error, and the initialization error starting from the second stage.
\begin{equation}
\epsilon_{n+1}  =\sum^{n}_{i=1}I({\theta^*_{i,0},\epsilon_{i-1}})+ \sum^{n}_{i=0}\delta_{i+1} +\epsilon_{0}.
\end{equation}
\end{theorem}

\begin{proof}
For the stage directly matched in the condensation process, we assume that its matching error can be reduced to a negligible value. Assuming the sum of matching errors for the remaining segments is $\mu$.

\begin{equation} \label{eqn-18}
\begin{split}
\epsilon_{n+1}  &= \hat{\theta}_{n,q} -  \theta^{*}_{n,p}\\
&=(\hat{\theta}_{n,0} + \Theta_S(\hat{\theta}_{n,0},q))
- (\theta^{*}_{n,0}+\Theta_T(\theta^{*}_{n,0},p))\\
&=
(\hat{\theta}_{n,0} + \Theta_S(\theta^{*}_{n,0}+\epsilon_{n},q))
- (\theta^{*}_{n,0}+\Theta_T(\theta^{*}_{n,0},p))\\
&=
(\hat{\theta}_{n,0}- \theta^{*}_{n,0})+(
 \Theta_S(\theta^{*}_{n,0}+\epsilon_{n},q)
-\Theta_S(\theta^{*}_{n,0},q))
+(\Theta_S(\theta^{*}_{n,0},q)
- 
\Theta_T(\theta^{*}_{n,0},p))
\\
&=
\epsilon_{n}  + I({\theta^*_{n,0},\epsilon_t})+\delta_{n+1}
\\
&=\sum^{n}_{i=1}I({\theta^*_{i,0},\epsilon_{i-1}})+ \sum^{n}_{i=0}\delta_{i+1} +\epsilon_{0},
\end{split}
\end{equation}


As shown in Equation \ref{eqn-18}, the accumulated error during the evaluation process can be represented as the result of the summation of the initial accumulated error, the sum of the matching errors, and the accumulated errors, except for those that have been reduced.
\end{proof}

\subsection{Detailed Analysis}
In the meta-matching method proposed by SFGC, only one segment of the expert trajectory is selected for trajectory matching. This approach can only utilize a small part of the information in the whole training trajectory. In our method, there are two improvements in the condensation matching phase. 
Firstly, we use an expanding window starting from 0 for matching, which means during the \textit{condensation} phase, more stages will be matched, resulting in a smaller matching error $\mu^{'}$,
and the student network is trained multiple times starting from \(\theta^*_{0,0}\), thus incorporating \(\epsilon_{0}\) into the optimization. Note that when \(t=0\), the definitions of accumulated error and matching error are the same, and due to the optimization of \(\epsilon_{0}\), \(I(\theta^*_{1,0},\epsilon_{0})\) is also optimized simultaneously~\cite{du2023minimizing}.
More importantly, previous research has shown that curriculum learning can generate flatter training trajectories~\cite{sinha2020curriculum, sitawarin2021sat, krishnapriyan2021characterizing, zhang2021curriculum}, which can optimize \(I(\theta^*_{n,0},\epsilon_{n-1})\) efficiently~\cite{du2023minimizing}. 

We denote the optimized accumulated error and initialization error as \(\epsilon^{'}_{n}\) and \(I'\) respectively.
Assuming $||I(\theta^*_{1,0},\epsilon_{0}) - I'(\theta^*_{1,0},\epsilon_{0})||=\tau_1>0$,  $||\epsilon_{0} - \epsilon^{'}_{0}||\geq\tau_2>0$,  $||I(\theta^*_{i,0},\epsilon_{i-1}) - I'(\theta^*_{i,0},\epsilon_{i-1})||=\tau_3>0$ and $||\mu - \mu^{'}||=\tau_4>0$, we have
\begin{equation}
\epsilon^{*'}_{n} = \epsilon_{n} - \tau_1 - \tau_2 - (t^{*}-2)\tau_3 - \tau_4< \epsilon^{*}_{n}.
\end{equation}
This implies that our method has a smaller accumulated error during the \textit{evaluation} phase, resulting in better performances.

\section{Training Samples Analysis}\label{section: samples}
In our quest to identify nodes that play a dominant role in the formation of expert trajectories, we train each node in the training set sequentially through a GNN and record the gradient values generated by each node. At the same time, based on a ranking of difficulty, we classify the lowest 70\% of nodes in terms of difficulty scores as easy nodes and the highest 30\% as difficult nodes and compute the average gradients for easy and difficult nodes. 
As illustrated in Fig.~\ref{fig:arxiv grad of SFGC}, due to the challenge GNNs face in learning clear representations from these difficult nodes, larger gradients are produced, and the GNN tends to focus more on these nodes during training. Consequently, in expert trajectories, the supervision signals from difficult nodes are more emphasized. 

To explore the distinct guiding roles of these nodes during the condensation phase, we train the GNN with different ratios of easy to difficult nodes (since we need to control the ratio while maintaining a consistent number of nodes used, we must select subsets from the whole training set) to form expert trajectories. As shown in Fig.~\ref{fig:Performance of different nodes1} and~\ref{fig:Performance of different nodes2}, the supervision role of easy nodes is essential for optimizing the condensed graph; relying solely on difficult nodes is insufficient for optimization, as they rarely contain the general patterns of the original graph.

\section{Time Complexity Analysis}
\textbf{Time complexity}. 
We first measure the time complexity of the GCond, SFGC, and GEOM.
For simplicity, let the number of GCN layers in the adjacency matrix generator be $L$, the number of the sampled neighbors per node be $r$ and all the hidden units be $d$. The number of nodes in the original graph dataset is $N$ and the number of nodes in the condensed graph dataset is $N^{'}$. 
In the forward process, training GCN on the original graph has a complexity of $O(r^{L}Nd^{2})$ and $O(LN^{\prime2}d+LN^{\prime}d)$ on the condensed graph.
For GCond, it has time complexity of $TKO(LN^{\prime2}d+LN^{\prime}d)+TKO(N^{\prime2}d^{2})$, where $\zeta$ denotes the number of outermost loops, $K$ denotes the number of different initialization.
For SFGC and our method, the time complexity of matching training trajectories is about $TO(LN^{\prime}d^{2}+LN^{\prime}d)$, and offline training expert GCNs have a time complexity of $MO(r^{L}Nd^{2})$, where $M$ is the number of the experts. 
Although our method requires calculating the loss generated by a one-time forward on the expert GCN on the condensed graph, the small size of the condensed graph means that the time for a single forward on a 2-layer GCN is almost negligible. Additionally, the expert parameters do not require extra training to obtain, so we consider the time complexity of this operation as a constant $E$.

\textbf{Running time}.
Although our method introduces an additional constant $E$ in terms of time complexity, we have made improvements in how the condensed graph is evaluated, saving time, especially in the evaluation of larger-scale condensed graphs, thereby further enhancing the efficiency of our method. Concretely, the assessment of condensed graph-free data involves training a GNN model with it. The improved test performance of the GNN model in node classification at a particular meta-matching step suggests superior quality of the condensed data at that stage. 

Consequently, this evaluation process requires training a GNN model from the ground up while evaluating the GNN's performance at each epoch of training, which in turn incurs increased time and computational cost.
To lower the computational cost, SFGC choose to use a Graph Neural Feature Score to evaluate the condensed graph. However, the Graph Neural Feature Score can only work in an extremely low condensation ratio, the greater the condensation ratio, the less pronounced the advantages of the Graph Neural Feature Score become.

Aiming to enhance the efficiency of evaluating the condensed graph, we analyze the root causes of efficiency issues with the original evaluation method. Firstly, since the size of the condensed graph is much smaller compared to the original graph dataset, as we mentioned in our time complexity analysis, the time taken for the condensed graph to forward on a GCN is very short. The majority of the time spent in the original method of training a GNN from scratch was due to testing on the original dataset's test set in each training epoch, to determine the best-performing training epoch for the final performance. 

Therefore, we design a short-term interval training evaluation to assess the performance of the condensed graph. Specifically, we do not require the GNN trained with the condensed graph to reach a well-trained state, but instead, we try to reduce the number of training epochs (e.g., training only for 200 epochs on Ogbn-arxiv) and compare the performance of condensed graphs within the same training epochs. 
Also, during training on condensed graphs, we do not evaluate the GNN at every epoch but do so after a considerable number of training intervals, e.g., an interval of 20 epochs. 
By adopting this approach, we significantly reduce computational time and enhance the efficiency of evaluating the condensed graph. All experiments are conducted five times on one single Nvidia-A100-SXM4-80GB GPU. We provide the running time of our method and SFGC in Fig.~\ref{fig:time}.

\vspace{-5.0pt}
\begin{figure}[ht]
	\centering
    \subfigure[]
     {\includegraphics[width=0.4\textwidth, angle=0]{ 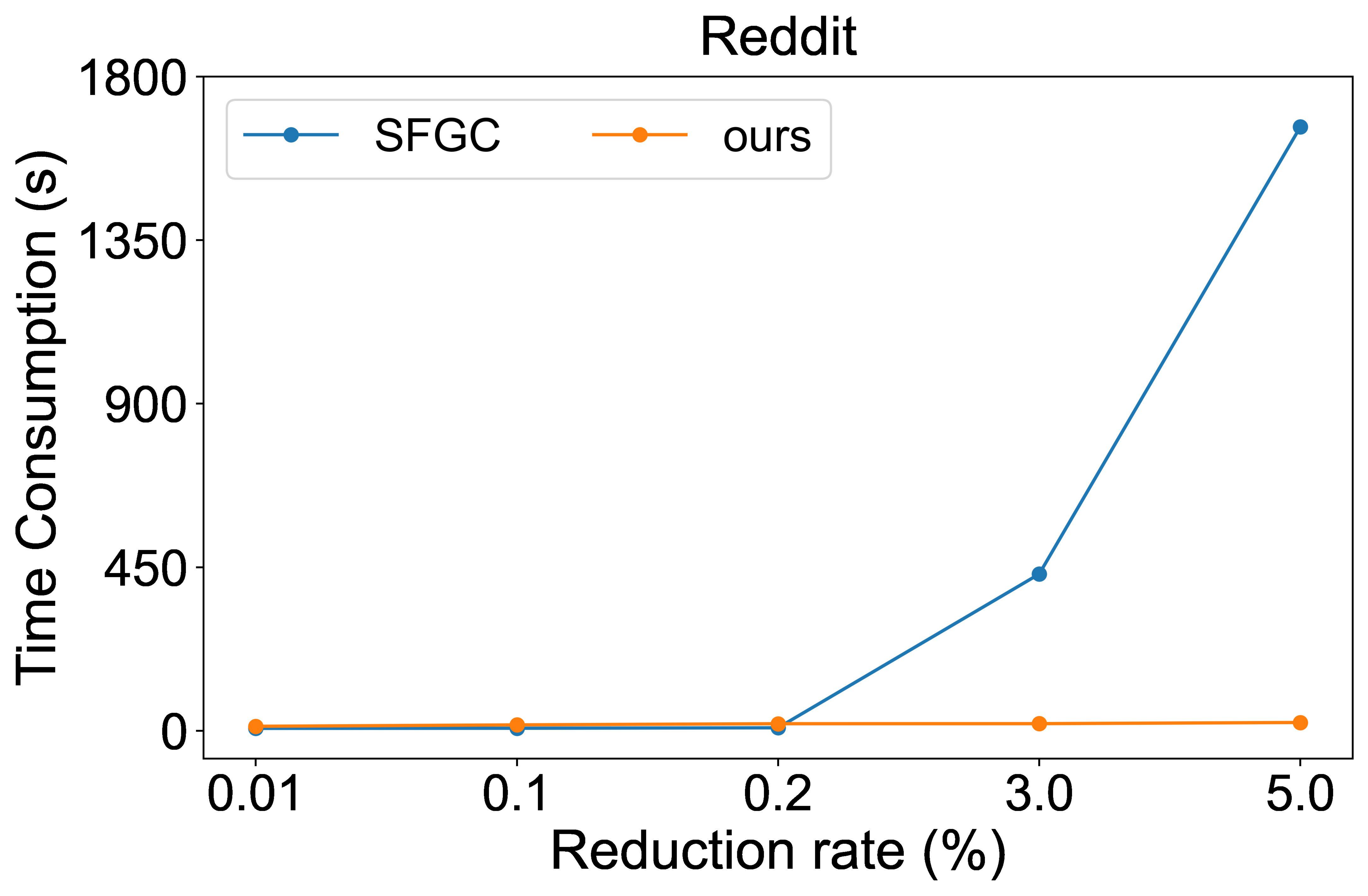}}
     \hspace{5mm}
     \subfigure[]
     {\includegraphics[width=0.4\textwidth, angle=0]{ 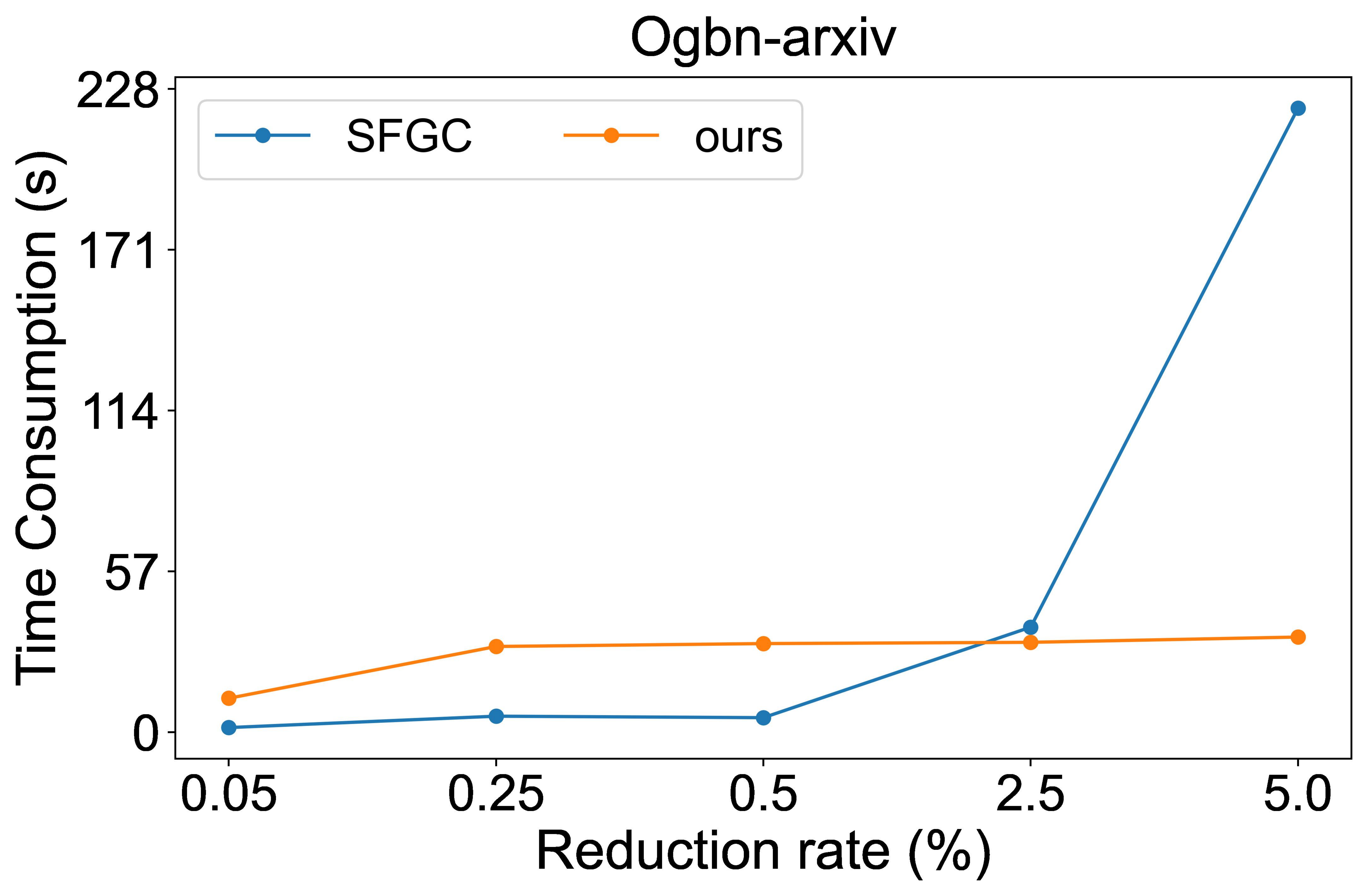}}
    \vspace{-8.0pt}
	\caption{Comparison of methods for evaluating and storing condensed graphs.}\label{fig:time}
    \vspace{-10.0pt}
\end{figure}
\vspace{-3.0pt}

\section{Analysis of Matching Range}
In exploring the effects of different ranges of long-term matching, we present the impacts of various step combinations of $q$ steps (student) and $p$ steps (expert) on the Ogbn-arxiv dataset, with r = 0.5\%. The results, displayed in Fig.~\ref{fig:synstep}, show that the optimal step combination exists for 2100 student steps ($q$) and 1900 expert steps ($p$). Under this setup, the condensed graph exhibits the best node classification performance. Additionally, the quality and expressiveness of the condensed data moderately vary with different step combinations, but the variance is not overly drastic.

Moreover, regarding the different step combinations of $p$ and $q$, we observe that without using soft labels, GEOM exhibits properties similar to SFGC, where the optimal value of $q$ is usually smaller. 
In the choice of $p$, due to the adoption of a curriculum learning approach and expanding window matching, a smaller $p$ can often be set during the condensation. 

In cases where soft labels are used, we find that increasing $q$ under the same $p$ settings generally yields better results. 
One potential reason is that the information in soft labels is more complex compared to hard labels and requires more optimization steps~\cite{guo2023towards}. 
Concurrently, we eliminate unnecessary storage of student model parameters, thereby avoiding excessive memory demands caused by increasing $q$.

\begin{figure}[ht]
	\centering
     {\includegraphics[width=0.45\textwidth, angle=0]{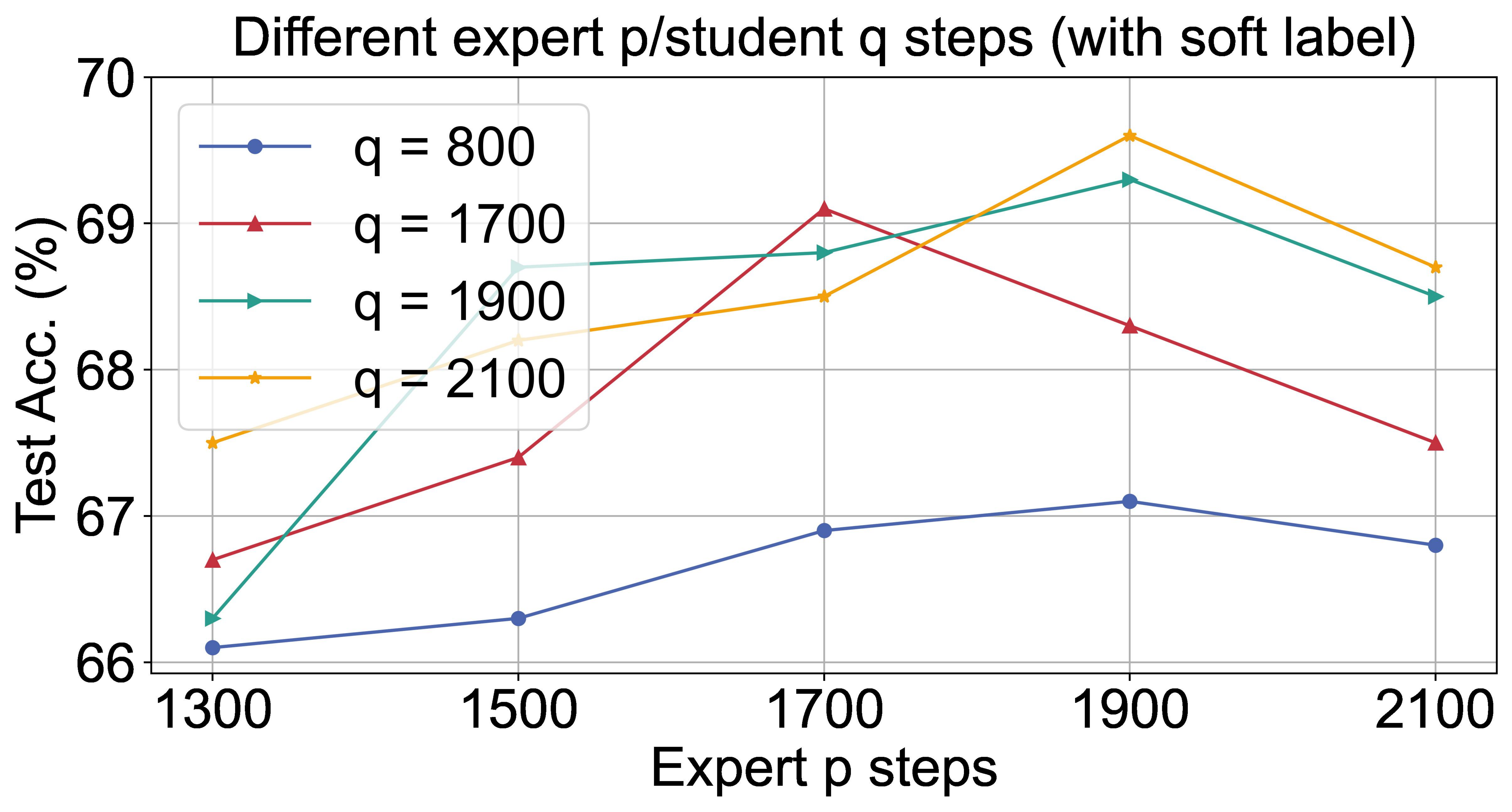}
    \label{fig:soft}}
    \hspace{1mm}
    {\includegraphics[width=0.45\textwidth, angle=0]{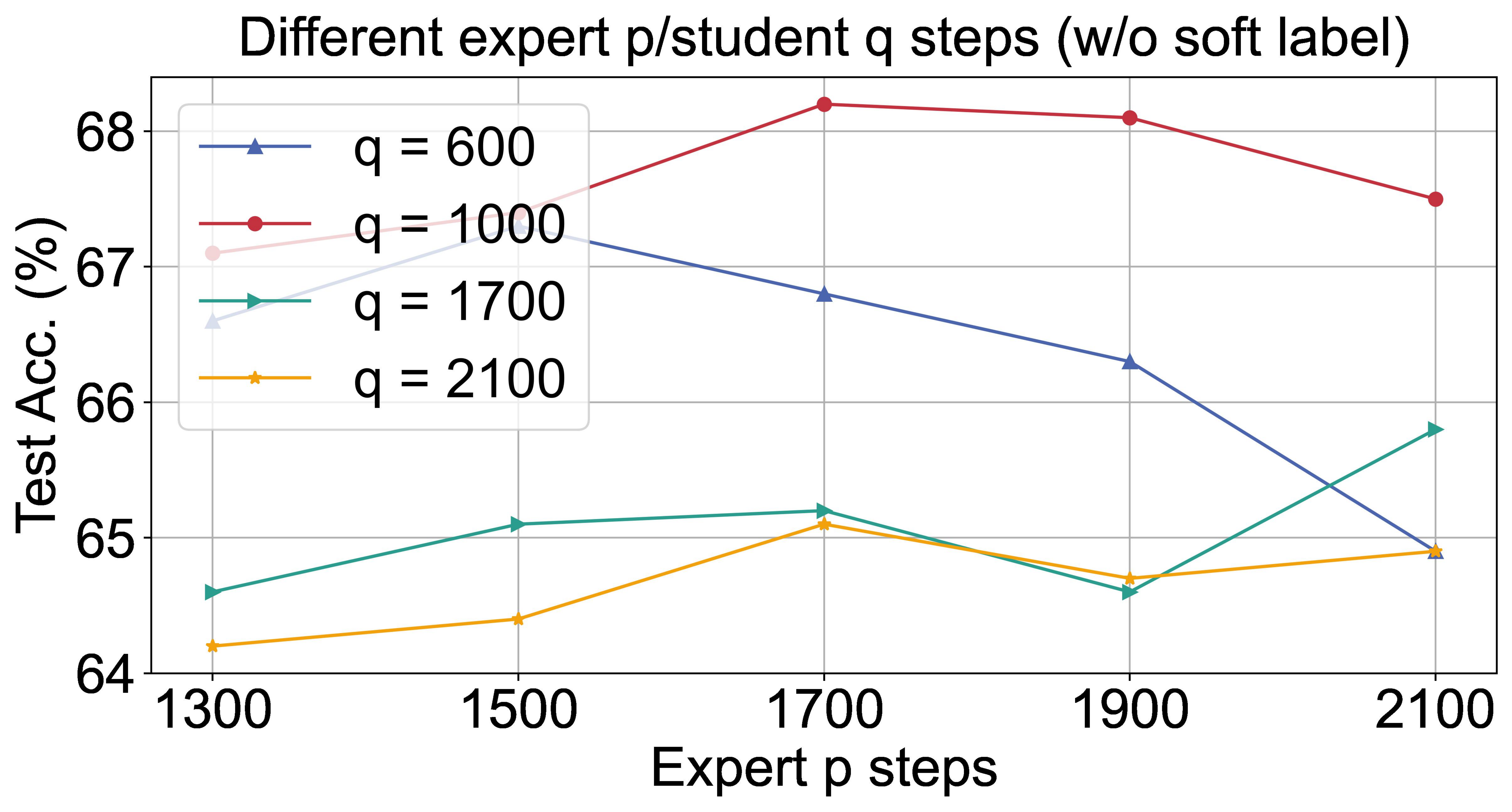}
    \label{fig:hard}}
    \hspace{1mm}
	\caption{Performance with different step combinations of $q$ student steps and expert $p$ steps on
Ogbn-arxiv ($r$ = 0.5\%).}\label{fig:synstep}
    \vspace{-8.0pt}
\end{figure}
\section{Implementation Details}\label{section: details}
For the condensation ratio ($r$) choices, we adhere to the settings from previous studies for smaller datasets such as Cora, Citeseer, and Flickr, where our method effortlessly achieves lossless compression. Specifically, we choose {1.30\%, 2.60\%, 5.20\%} for Cora, {0.90\%, 1.80\%, 3.60\%} for Citeseer, and {0.10\%, 0.50\%, 0.10\%} for Flickr.
However, for the Ogbn-arxiv and Reddit datasets, we found that the previously set condensation ratios were insufficient to involve enough information to get lossless~\cite{jin2022graph}. Therefore, after conducting numerous experiments, we introduce two additional sets of condensation experiment settings for these two datasets. Specifically, we choose {0.10\%, 0.50\%, 1.00\%, 2.50\%, 5.00\%} for Ogbn-arxiv and {0.05\%, 0.10\%, 0.20\%, 3.00\%, 5.00\%} for Reddit.

In the process of training an expert trajectory, we primarily adjust three parameters to control the process of incorporating simple and difficult information: the number of epochs for training on the entire dataset ($\zeta$), the initial proportion of easy nodes ($\lambda$), and the method of gradually adding difficult nodes to the training data (Scheduler).
It is worth noting that our curriculum learning approach can not improve the final performance obviously; rather, it focuses on obtaining trajectories that include clearer and more diverse information from the original graph.

During the \textit{condensation} phase, we build upon SFGC by introducing parameters to control the expanding window and KEE, without specific mention, we adopt a 2-layer GCN with 256 hidden units as the GNN used for condensation. All other parameters remain consistent with those publicly disclosed for SFGC. The specific parameter settings are outlined in Table \ref{tab:paras}, where  ${U^{'}}$ represent the upper bounds of the expanding window, and  ${U}$ denotes the upper limit of the initial expanding window, incremented by one after each condensation iteration. Notably, lr\_y set to 0 indicates the absence of soft labels. An important experimental observation is that omitting early trajectory information across all datasets leads to suboptimal results. Consequently, we set the start of the expanding window to 0 consistently.

In practical implementation, we observe that soft labels can sometimes lead to optimization process instability, especially for certain small-scale condensed datasets. In such experimental scenarios, we use hard labels for the KEE process. 
Therefore, we do not introduce additional loss computations on this dataset.

It is important to highlight that for condensed graphs derived from Reddit and Ogbn-arxiv with condensation ratios greater than 1\%, achieving optimal results requires fewer optimization iterations. A possible explanation is that when the scale of the condensed graph is larger, the gap between it and the original data can be bridged with relatively minor adjustments.

In the selection of methods for evaluating and storing condensed datasets, we don't use the graph kernel-based method (GNTK) proposed by SFGC. This is because as the scale of condensed graphs increases, the computational time for the GNTK metric grows exponentially. When the scale of the condensed dataset is large, the time consumed to compute this metric is about six times that of training a GNN directly with the condensed graph, as illustrated in Fig.~\ref{fig:time}. Noting that to achieve a fairer comparison, we use different random seeds for the evaluation function when choosing the condensed datasets to save during the \textit{condensation} phase and when assessing the performance of the condensed graphs after condensation.

\begin{table*}[ht]
  \centering
  \caption{Hyper-parameters of the overall framework}
  \label{tab:paras}
  \scriptsize
  \begin{tabular}{ccccccccccccc}
    \toprule
    Dataset & Ratio & $\zeta$ & $\lambda$ & Scheduler & $U^{'}$ & $U$  & $p$ & $q$ & lr\_feat & lr\_y & $\alpha$ \\
    \midrule
    \multirow{3}{*}{Citeseer} & 0.90\% & 250 & 0.8 & root & 30 & 20 & 350 & 200 & 0.0001 & 0 & 0.1 \\
    & 1.80\% & 250 & 0.8 & root & 80 & 20 & 350 & 200 & 0.0007 & 0 & 0.05 \\
    & 3.60\% & 250 & 0.8 & root & 30 & 20 & 350 & 400 & 0.0001 & 0 & 0.1 \\
    \midrule
    \multirow{3}{*}{Cora} & 1.30\% & 1500 & 0.75 & geom & 200 & 50 & 1400 & 2500 & 0.0001 & 0.00005 & 0.01 \\
    & 2.60\% & 1500 & 0.75 & geom & 200 & 50 & 1400 & 2500 & 0.0001 & 0.00005 & 0.01 \\
    & 5.20\% & 1500 & 0.75 & geom & 200 & 50 & 1500 & 2500 & 0.0001 & 0.00005 & 0.01 \\
    \midrule
    \multirow{5}{*}{ogbn-arxiv} & 0.05\% & 1200 & 0.85 & root & 50 & 30 & 1100 & 650 & 0.25 & 0 & 0 \\
    & 0.25\% & 1200 & 0.85 & root & 200 & 100 & 1600 & 2100 & 0.05 & 0.001 & 0 \\
    & 0.50\% & 1200 & 0.85 & root & 200 & 100 & 1900 & 2100 & 0.03 & 0.001 & 0 \\
    & 2.50\% & 1200 & 0.85 & root & 350 & 300 & 1600 & 2200 & 0.03 & 0.001 & 0 \\
    & 5\% & 1200 & 0.85 & root & 400 & 300 & 1500 & 2000 & 0.05 & 0.001 & 0 \\
    \midrule
    \multirow{3}{*}{Flickr} & 0.10\% & 100 & 0.95 & root & 30 & 10 & 600 & 600 & 0.07 & 0 & 0.3 \\
    & 0.50\% & 100 & 0.95 & root & 30 & 1 & 600 & 300 & 0.01 & 0 & 0.1 \\
    & 1\% & 100 & 0.95 & root & 70 & 10 & 70 & 300 & 0.07 & 0 & 0.3 \\
    \midrule
    \multirow{5}{*}{Reddit} & 0.05\% & 800 & 0.9 & linear & 50 & 1 & 800 & 800 & 0.02 & 0 & 0.25 \\
    & 0.10\% & 800 & 0.9 & linear & 20 & 1 & 1000 & 1000 & 0.03 & 0 & 0.1 \\
    & 0.20\% & 800 & 0.9 & linear & 20 & 1 & 900 & 800 & 0.2 & 0 & 0.25 \\
    & 3\% & 800 & 0.9 & linear & 250 & 50 & 900 & 1300 & 0.001 & 0.0001 & 0.2 \\
    & 5\% & 800 & 0.9 & linear & 200 & 1 & 1000 & 1300 & 0.001 & 0.0001 & 0.25 \\
    \bottomrule
  \end{tabular}%
\end{table*}

\section{Visualizations}
\label{appendix_vis}
We showcase t-SNE plots depicting the condensed graph-free data generated by GEOM across all datasets. Our condensed graph-free data reveals a well-clustered pattern across Cora and Citeseer. Furthermore, larger-scale datasets exhibit some implicit clusters within the same class. This indicates that our approach effectively learns representative representations from the easy nodes of the original data while efficiently utilizing the difficult nodes. With the assistance of difficult nodes, the patterns become enriched.

\begin{figure}[ht]
	\centering
    \subfigure[Flickr, $r=1\%$]
    {\includegraphics[width=0.4\textwidth, angle=0]{ 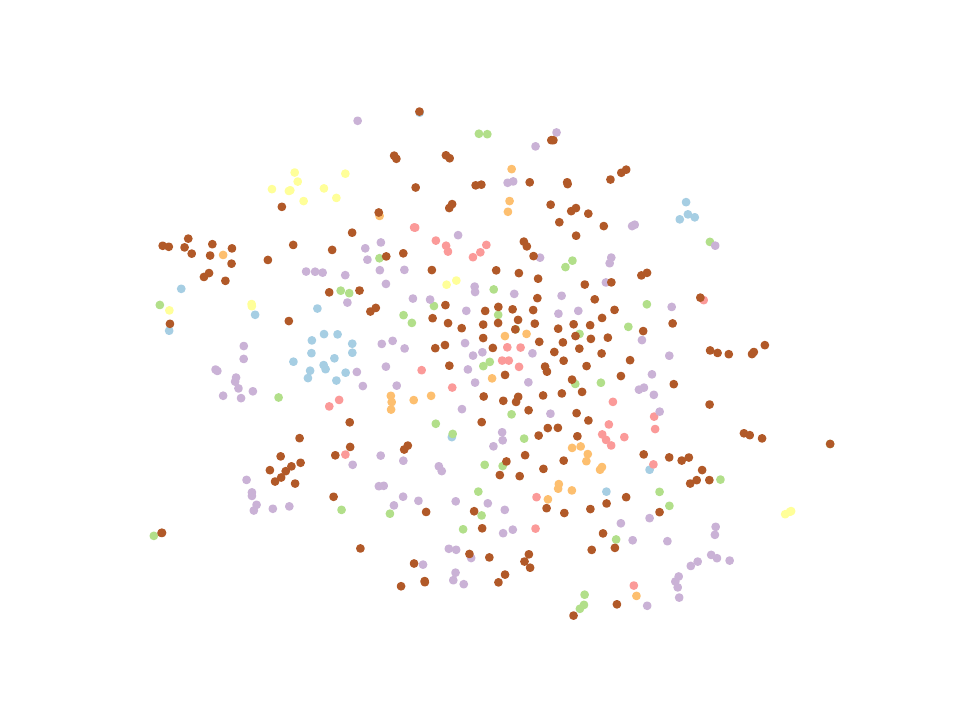} 
  }
    \hspace{1mm}
        \subfigure[Ogbn-arxiv, $r=0.5\%$]
    {\includegraphics[width=0.4\textwidth, angle=0]{ 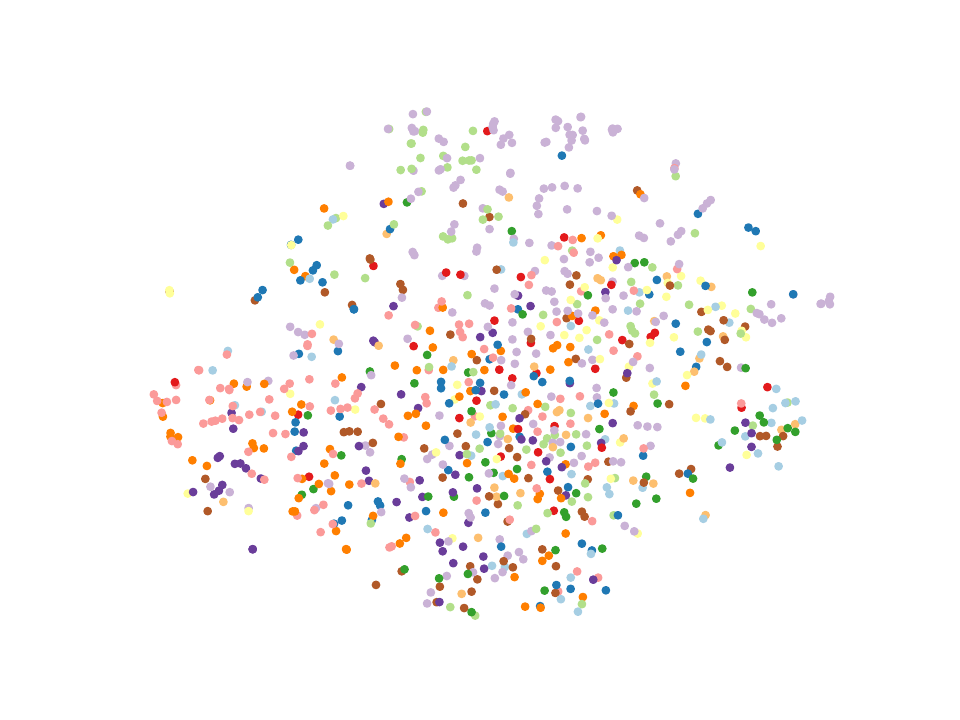} 
   }
    \hspace{1mm}
    \subfigure[Ogbn-arxiv, $r=5\%$]
    {\includegraphics[width=0.4\textwidth, angle=0]{ 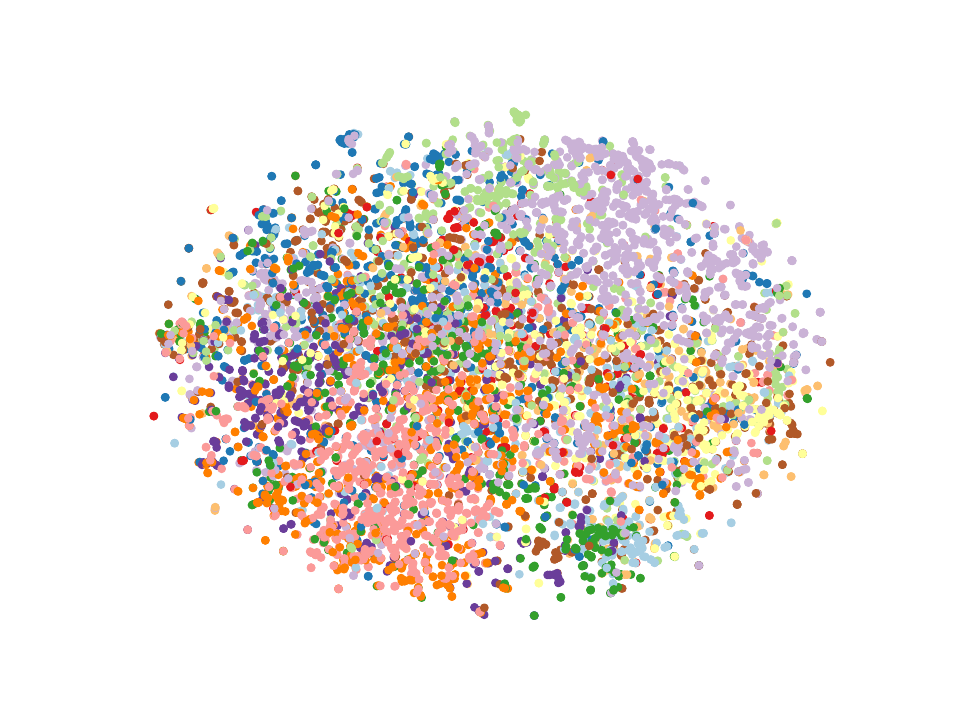} 
 }
    \hspace{1mm}
        \subfigure[Reddit, $r=5\%$]
    {\includegraphics[width=0.4\textwidth, angle=0]{ 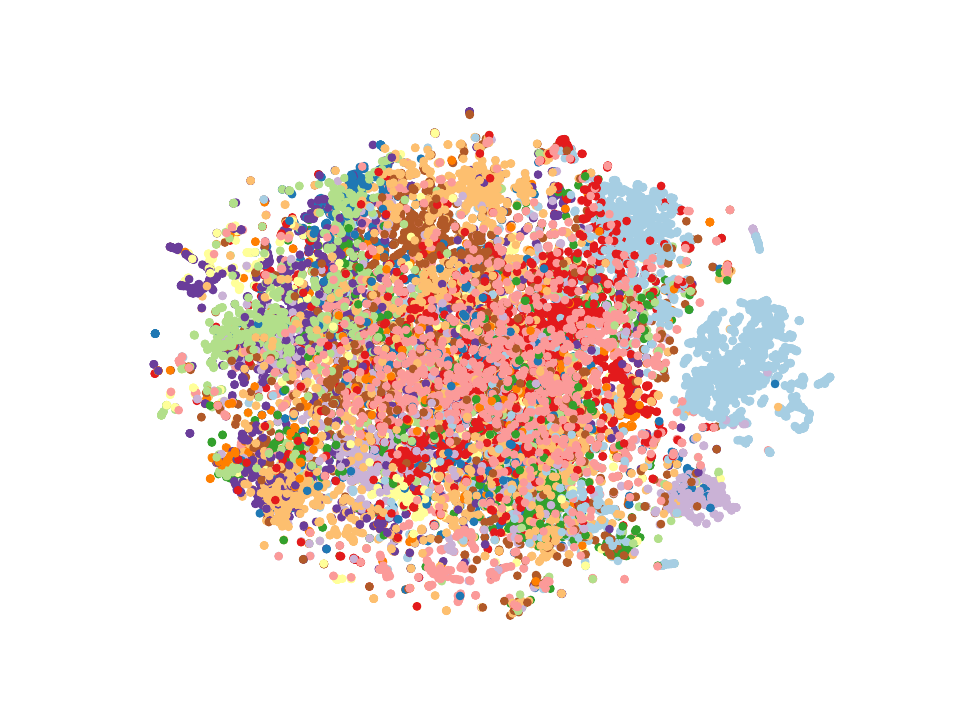} 
 }
    \hspace{1mm}
	\caption{Visualization of t-SNE on condensed graphs}\label{fig:vis_syngraph_all}
    \vspace{-8.0pt}
\end{figure}


\end{document}